\documentclass{article} 
\usepackage{iclr2027_conference,times}

\usepackage[utf8]{inputenc} 
\usepackage[T1]{fontenc}    
\usepackage{hyperref}       
\usepackage{url}            
\usepackage{booktabs}       
\usepackage{amsfonts}       
\usepackage{nicefrac}       
\usepackage{microtype}      
\usepackage{xcolor}         

\title{Verification of Unknown Dynamical Systems via Autoencoder Latent Space}

\author{%
  Robert Reed\\
  Dept. of Aerospace Engineering\\
  University of Colorado Boulder\\
  Boulder, Colorado, USA\\
  \texttt{robert.reed-1@colorado.edu} \\
  \And
  Luca Laurenti\\
  Delft University of Technology\\
  Delft, Netherlands\\
  AI4I, Turin, Italy\\
  \texttt{L.Laurenti@tudelft.nl} \\
  \And
  Morteza Lahijanian\\
  Dept. of Aerospace Engineering\\
  University of Colorado Boulder\\
  Boulder, Colorado, USA\\
  \texttt{morteza.lahijanian@colorado.edu} \\
}

\usepackage{amsmath,amssymb,amsfonts}
\usepackage{bbold}
\DeclareMathAlphabet{\mathmybb}{U}{bbold}{m}{n}
\newcommand{\ind}{\mathmybb{1}}

\usepackage{graphicx}
\usepackage{xcolor}

\usepackage{multirow}

\usepackage{subcaption}

\usepackage{hyperref}
\usepackage{xspace}
\usepackage{enumitem}

\newtheorem{theorem}{Theorem}
\newtheorem{problem}{Problem}

\newtheorem{proposition}{Proposition}

\newtheorem{lemma}{Lemma}
\newtheorem{definition}{Definition}
\newtheorem{remark}{Remark}

\newcommand{\px}{\mathbf{x}}

\newcommand{\pv}{\mathbf{v}}
\newcommand{\py}{\mathbf{y}}
\newcommand{\prop}{\mathrm{p}}

\newcommand{\U}{\mathcal{U}}
\newcommand{\X}{\mathcal{X}}
\newcommand{\Z}{\mathcal{Z}}

\newcommand{\D}{\mathcal{D}}

\newcommand{\reals}{\mathbb{R}}
\newcommand{\nats}{\mathbb{N}}
\newcommand{\expect}{\mathbb{E}}
\newcommand{\enc}{h_{e}}
\newcommand{\dec}{h_{d}}

\newcommand{\post}{\mathrm{Post}}
\newcommand{\pre}{\mathrm{Pre}}
\newcommand{\conv}{\mathrm{Conv}}

\DeclareMathOperator*{\argmin}{\arg \min}

\iclrfinalcopy 

\begin{document}

\maketitle
\makeatletter
\lhead{Preprint. Work in progress.}
\makeatother

\begin{abstract}
    Formal verification provides a powerful framework for proving that dynamical systems satisfy their specifications. However, these techniques face scalability challenges in high-dimensional settings, as they often rely on state-space discretization which grows exponentially with dimension. Learning-based approaches to dimensionality reduction have shown great potential to alleviate this problem. However, ensuring correctness of latent space verification results remains an open question. In this work, we provide a formal approach to reduce the dimensionality of systems via convex autoencoders and learn the dynamics in the latent space through a kernel-based method. We then construct a finite abstraction from the learned model in the latent space and guarantee that the abstraction contains the true behaviors of the original system. We show that the verification results in the latent space can be mapped back to the original system. Finally, we demonstrate the approach on multiple systems, including a 26D system controlled by a neural network, showing significant scalability improvements.
\end{abstract}

\section{Introduction}
    \label{sec:intro}
    
Formal verification of continuous-space dynamical systems provides a rigorous way to guarantee that a system satisfies a desired specification~\citep{baier2008principles, Tabuada2009, belta2017formal}. A standard approach is to express the specification in Linear Temporal Logic (LTL)~\citep{baier2008principles} and construct a finite-state abstraction that captures all possible behaviors of the system. Off-the-shelf model-checking tools can then be applied to this abstraction to determine whether the specification is satisfied.
This approach, however, faces significant challenges when applied to \emph{AI-enabled} systems, such as those controlled by neural networks that operate on \emph{high-dimensional} sensor inputs (e.g., camera images or LiDAR). First, the size of the abstraction grows \emph{exponentially with the dimension} of the state space. Second, for systems with \emph{complex} or partially \emph{unknown} dynamics, computing the transition relation between abstract states becomes highly nontrivial. Yet these challenges are increasingly critical, as AI components are rapidly becoming integral to modern autonomous systems.
In this work, we aim to address these challenges by developing a formal verification framework for systems with unknown dynamics that scales to high-dimensional settings.

Model order reduction has been widely used to simplify systems for classical control and verification~\citep{obinata2012model, baur2014model, antoulas2001survey}.
However, for systems with highly complex or unknown dynamics, such reduction techniques often fail to generalize. To address this limitation, modern approaches employ data-driven techniques, such as autoencoders~\citep{bengio2013representation}, to learn latent representations and latent dynamics for control, typically trading correctness guarantees for flexibility. For systems with unknown dynamics, some works construct finite abstractions with correctness guarantees~\citep{jackson2021formal, reed2023promises}, but these approaches suffer from scalability limitations. Applying such abstraction techniques directly in a learned latent space is challenging because existing frameworks assume explicit, well-defined system dynamics, whereas latent-space dynamics induced by learned representations may instead be governed by \emph{inclusion} dynamics, violating these assumptions.

To the best of our knowledge, no existing method enables formal verification in a latent space 
\emph{while also providing correctness guarantees} for both the mapping of specifications into the latent representation and the learned latent dynamics.
This work bridges this gap by combining dimensionality reduction with abstraction-based verification to provide formal guarantees for high-dimensional systems with unknown dynamics.
Our framework is built on two key observations: (i) the learned latent space must correctly represent the regions of interest in the state space, and (ii) the latent dynamics may naturally follow an inclusion model. To address these challenges, we employ convex autoencoders to map the high-dimensional space into a latent space with provable guarantees and develop a novel regression technique that enables sound reasoning about inclusion dynamics. 

The contributions of this work are four-fold:
(i) a provably correct mapping of high-level specifications into a latent space, enabling efficient reasoning over the compressed representation,
(ii) a novel formulation of Gaussian Process (GP) regression for inclusion-valued (set-valued) dynamics, which we call \emph{Inclusion GPs}, enabling learning of latent dynamics governed by difference inclusions,
(iii) a framework for generating sound finite abstractions in the latent space, based on Inclusion GPs, along with procedures for 
mapping satisfaction guarantees back to the original high-dimensional system with proof of correctness,
(iv) demonstration of the efficacy of our approach on multiple systems, including a 26-dimensional system controlled by a neural network.

\textbf{Related Work \ }
%
Many works explore the use of a latent space for planning and control. A common focus is on enabling control for high-dimensional systems such as humanoid robotics or systems that derive control from visual observations. Common techniques for control in the latent space include model predictive control~\citep{watter2015embed, masti2021learning}, reinforcement learning~\citep{hafner2019learning}, and Hamilton-Jacobi reachability~\citep{nakamura2025generalizing}.
While these techniques often have great empirical success, they typically lack guarantees on the correctness of the approach which can result in false positives.  

Several recent works attempt to provide formal guarantees for latent-space control; however, they often rely on strong assumptions about the dynamics or demonstrate only empirical success. In particular, 
\citet{nakamura2025generalizing, vieira2024morals, lee2024mission} make use of autoencoders to define a latent space for control, but lack guarantees on their learned latent space dynamics. \citet{chen2016decomposed, awan2025reduced, schon2025formal} propose methods to soundly reason about dynamics in a low-dimensional space, but
have strong assumptions on the form of the dynamics. \citet{lutkus2025latent} develops a theory to correctly reason about dynamics in a latent space, but has no method to soundly define the latent domain, preventing the construction of a sound abstraction. We should also stress that various works make use of PCA to identify a low dimensional representation a system for control~\citep{matrone2010principal, dalibard2011linear}, but as noted by \cite{poupart2002value, bishop2006pattern} without significant structural constraints on the high dimensional representation (e.g., linear dynamics) the latent space identified by PCA is unlikely to preserve the dynamics exactly. 

Unlike the aforementioned literature, this work provides soundness guarantees. Specifically, we construct an interpretable latent space for arbitrary continuous systems with correctness guarantees on the learned latent dynamics, hence eliminating false positives in the verification results.

\section{Problem Formulation}
    \label{sec:problem}
    Consider a discrete-time dynamical system: 
\begin{align}
    \label{true_dynamics} 
    \px(k + 1) &= f(\px(k)) 
\end{align}
where $\px(k) \in \reals^{n_x}$ is a fully-observable, possibly \emph{high-dimensional} state, and $f: \reals^{n_x} \rightarrow \reals^{n_x}$ is a 
possibly \emph{unknown} function. 
Intuitively, System~\eqref{true_dynamics} represents an autonomous system operating under a closed-loop controller that may be partially or entirely unknown due to AI (black-box) components,
e.g., System~\eqref{true_dynamics} can represent a system $s(k+1) = \mathtt{f}(s(k), \pi(y(k)))$ under observation-dependent policy/controller $\pi$ with input measurements $y(k) = \mathrm{o}(s(k))$, where $\mathrm{o}(\cdot)$ is some observation function of $s$, 
by defining $\px(k) = (s(k), y(k))$ as the stacked vector of $s$ and $y$. A grounded example is a robot navigating a room using a neural network controller $\pi$ that relies on LiDAR measurements $y$ and a desired goal orientation, where $\px$ is the stack vector of the robot's state $s$ and LiDAR readings $y$ necessary to reason about closed-loop system, making dimension $n_x$ potentially very large.

A \textit{trajectory} of System \eqref{true_dynamics} is written as $\omega_\px = \px(0)  \px(1) \ldots$ and the set of all trajectories is denoted as $\Omega_\px$.
Our aim is to verify the trajectories of System~\eqref{true_dynamics} against complex temporal requirements over regions in $\mathbb{R}^{n_x}$. These specifications, often expressed in temporal logic (e.g., LTL), reduce to \textit{reach-avoid} properties over an extended state space via a finite abstraction.  For simplicity of presentation, we focus on these properties without loss of generality. Given a compact domain $X \subset \mathbb{R}^{n_x}$, goal set $X_G \subset X$, and a set of unsafe regions $X_U = \{u_1, \ldots, u_l\}$, where $u_i \subset X$, we denote by $\varphi \equiv (X,X_G,X_U)$ a \textit{reach-avoid} specification, which requires reaching $X_G$ while staying within $X$ and avoiding $X_U$. 
We say a trajectory $\omega_\px$ satisfies 
$\varphi$, denoted by $\omega_\px \models \varphi$,
if $\exists i \in \nats$ s.t. $\omega_\px(i) \in X_G$ and $\forall j < i$, 
$\omega_\px(j) \in (X\setminus X_U$). 
In this work, we assume all sets $u_i \in X_U$, $X_G$, and $X$ are compact and convex. 
Given specification 
$\varphi$,
our goal is to identify the set of initial states, from which trajectories of System~\eqref{true_dynamics} 
satisfy 
$\varphi$.
However, since System~\eqref{true_dynamics} is possibly high-dimensional and unknown, we assume the availability of an i.i.d. dataset $D = \{({x}_i, {x}_i^+)\}_{i=1}^{m}$, where $x_i^+ = f(x_i)$.

\begin{problem} [Verification] \label{prblm:1} 
    Consider System~\eqref{true_dynamics} with a domain $X$, unsafe regions $X_U$, and goal region $X_G$, where these sets are compact and convex.
    Given dataset $D = \{({x}_i, {x}_i^+)\}_{i=1}^{m}$ of i.i.d. samples of System~\eqref{true_dynamics} and reach-avoid task $\varphi = (X, X_G, X_U)$, find a 
    maximal initial set $X_0 \subseteq X$ such that, for every $x_0 \in X_0$, the system is guaranteed to satisfy $\varphi$, i.e.,  $\omega_\px \models \varphi$ for every $\omega_\px(0) \in X_0$.
\end{problem}
%
%
While Problem~\ref{prblm:1} may resemble a classical verification problem, $f$ is both high-dimensional and unknown and this makes verification highly challenging. To address this challenge, a key element of our approach is to map the high-dimensional system to a lower-dimensional representation using an autoencoder. This reduction, however, must be performed carefully to ensure that the latent (low-dimensional) representation remains a sound abstraction of the original system, thereby preserving the desired guarantees. We guarantee a sound abstraction by restricting the class of functions that the autoencoder can represent and reasoning with inclusion dynamics in the latent space, as the mapping is in general not one-to-one. 

An overview of the approach is shown in Fig.~\ref{fig:CALM_GPR}.
Our method begins by learning an autoencoder that maps the system to a lower-dimensional latent space (Section~\ref{sec:autoencoder}). In this latent space, we construct a Nondeterministic Transition System (NTS) abstraction (Section~\ref{sec:NTS}) using (deep kernel) GP regression to accurately model the system dynamics (Section~\ref{sec:GP_modeling}). From the resulting NTS, existing verification techniques can be applied to determine which states satisfy $\varphi$.

\section{Background}
    \label{sec:background}

Here, we provide an overview of GP Regression (GPR) and the learning error bounds used in our method, and we introduce the key notation employed throughout our framework.

\textbf{Gaussian Process Regression and Error Bounds. \ }
Let $\D$ be a given dataset $\D = \{(x_i, y_i)\}_{i=1}^m$ obtained from sampling a function of the form $\py = \Psi(\px) + \pv$, with 
$y \in \reals$. GPR provides a method to predict the output of $\Psi$ at a new point $x^*$ using the dataset $\D$
and kernel $\kappa: \reals^{n_x} \times \reals^{n_x} \rightarrow \reals_{\geq 0}$
\citep{rasmussen:book:2006}.
Under the assumption that $\pv \sim \mathcal{N}(0, \sigma_n^2)$, a GP produces a posterior prediction of $\Psi$ conditioned on $\D$ with a mean $\mu_{\D}(x^*) = K_{x^*, \X}(K_{\X,\X} + \sigma_n^2 I)^{-1} Y$ and variance $\sigma^2_{\D}(x^*) = K_{x^*, x^*} - K_{x^*, \X}(K_{\X,\X} + \sigma_n^2 I)^{-1} K_{\X, x^*}$ 
where $\X = [x_1, \ldots, x_{m}]^T$ and $Y = [y_1, \ldots, y_m]^T$ are respectively the vectors of input and output data, $K_{\X,\X}$ is a matrix whose $i$-th row and $j$-th column is $\kappa(x_i, x_j)$, and $K_{x,\X} = K_{\X, x}^T$ are vectors whose $i$-th entry is $\kappa(x, x_i)$. 
Deep Kernel Learning (DKL)~\citep{Wilson2015} is an extension of GPR which incorporates a neural network $\psi$ into a base kernel as $\kappa(\psi(x), \psi(x'))$, enabling more accurate predictions~\citep{reed2023promises}. 
Under a standard smoothness assumption on $\Psi$, i.e., $\Psi$
lives in the Reproducing Kernel Hilbert Space (RKHS) of kernel $\kappa$, then there exists some constant $B > 0$ such that $\|\Psi(\cdot)\|_\kappa \leq B$ in the compact domain $X$. 
Under this assumption, 
\cite{reed2024error} shows that
GPR can be extended to the noise-free setting with regression error bounds according to the following proposition.
\begin{proposition}
    [\cite{reed2024error}, Theorem 2]
    Let $\Psi$ live in the RKHS of $\kappa$ with norm $B > 0$, and let $G = (K_{\X,\X} + \sigma_n^2I)^{-1}$, 
    $0 \leq d^* \leq \Psi(\X)^T G \Psi(\X)$. 
    In a noise-free setting (i.e., $\pv = 0$) then it holds that, for every $x \in X$ and with $\sigma_n >0$,
    $|\mu_{\D}(x) - \Psi(x)| \leq \sigma_{\D}(x)\sqrt{B^2 - d^*}.$
    \label{our_det_theorem}
\end{proposition}
\textbf{Notation. \ }
Given a high $n_x$-dim 
state $x$, an \emph{encoder} $\enc: \reals^{n_x} \rightarrow \reals^{n_p}$ maps $x$ to an $n_p$-dim 
state $z = \enc(x)$. 
We refer to $\reals^{n_p}$ as the \textit{latent space}.
A \emph{decoder} $\dec: \reals^{n_p} \rightarrow \reals^{n_x}$ then maps a $z \in \reals^{n_p}$ to a state $x$. 
Given a set of points $A \subset \reals^n$, $\conv(A)$ defines the convex hull of the set.  In a metric space $\reals^{n_x}$, $\mathcal{B}_{r}(x) = \{x' \in \reals^{n_x} \mid \|x-x'\| \leq r\}$ denotes a ball of radius $r \geq 0$ centered at $x \in \reals^{n_x}$.
For a function $h: \reals^n \rightarrow \reals^p$, 
$\post_{h}(A \subset \reals^n) := \{h(a) \mid a \in A\}$ denotes the \textit{posterior} of a $A$ under $h$, and 
$\pre_{h}(B \subset \reals^p) := \{b \in \reals^n \mid h(b) \in B\}$ is the \textit{pre-image} of a $B$ under $h$.

\section{Latent Space Abstraction
        }
    \label{sec:method}
    
In this section, we present our method of constructing a sound finite-abstraction in the latent space. It involves training an autoencoder coupled with DKL, i.e., GPR, in the latent space. Detailed proofs of all the lemmas and theorems are provided in Appendix~\ref{sec:app_proofs}. 
We first observe that when the latent space has a lower dimension than the original space, i.e., $n_p < n_x$, the resulting latent dynamics may be neither injective nor deterministic. 
In particular, when the dimensionality of System~\eqref{true_dynamics} is reduced beyond its minimal realization, i.e., the smallest number of parameters required to exactly reconstruct the observations, the evolution in the latent space is governed by \emph{inclusion dynamics}.

\vspace{-1mm}
\begin{definition}[Latent Inclusion Dynamics]
    \label{def: latent inclusion dynamics}
    Consider System~\eqref{true_dynamics} and encoder $\enc: \reals^{n_x} \to \reals^{n_p}$. 
    For a latent point $z \in \reals^{n_p}$, let $\X_z = \pre_{\enc}(z)$ be its pre-image and $\X'_z = \post_f(\X_z)$ be the posterior of $\X_z$ under the dynamics of System~\eqref{true_dynamics}.  
    Latent space \emph{inclusion dynamics} of System~\eqref{true_dynamics} are defined
    as
    \begin{align}
        \label{eq: latent space dynamics}
        z(k+1) \in g(z(k)),
    \end{align}
    where $g: \reals^{n_p} \to 2^{\reals^{n_p}}$ is given by $g(z(k)) =  \post_{\enc}(\X'_{z(k)})$. 
\end{definition}
If the latent space is a correct minimal realization, then $g(z(k))$ is a singleton. This is illustrated 
by an example in Appendix~\ref{app:example_inclusion}.  
Our second observation is that, generating a sound abstraction requires preserving the sets $X, X_U, X_G$ in the latent space, which we refer to as \textit{interpretability}.
As both the state and sets are defined for the original system, reducing dimensionality in an uninformed manner can cause significant difficulty in verifying $\varphi = (X, X_G, X_U)$. 
These two observations highlight the need for an encoder $\enc$ that produces an interpretable latent space and supports accurate learning of inclusion dynamics $g$. To achieve this, we impose structural constraints on $\enc$ and inform its training with the knowledge that latent evolution follows inclusion dynamics. 

\subsection{Encoder Architecture}
\label{sec:encoder architecture}

Our framework is designed around generating an interpretable latent space where the inclusion dynamics behave well. 
To this end, we first restrict the encoder to be an \emph{Input Convex Neural Network} (ICNN)~\citep{amos2017input}, enabling 
efficient mapping and
provable reasoning over latent regions. 
To 
regulate the behavior of inclusion dynamics in the latent space, we impose additional restrictions on the ICNN. Specifically, we (i) restrict activation functions to be strictly monotonic to be homeomorphisms, e.g., SoftPlus (see Remark~\ref{rem:monotone} in Appendix~\ref{app:remarks}), 
(ii) require the first layer and all skip layers to be injective, and 
(iii) enforce full-rank weight matrices in every layer. We refer to this architecture as a \emph{Connected Input-Connected Output} (CICO) network.
These constraints guarantee that the set $g(z(k))$ in Eq.~\eqref{eq: latent space dynamics} is a connected set when $f$ is continuous, as formalized in Lemma~\ref{lemma: connected pre and post image}.

\begin{lemma}[Connected Inclusion Dynamics]
    \label{lemma: connected pre and post image}
    Consider System~\eqref{true_dynamics}, neural network encoder $\enc: \reals^{n_x} \to \reals^{n_p}$, and the corresponding latent inclusion dynamics $g$ in Eq.~\eqref{eq: latent space dynamics}. 
    If $\enc$ is CICO, then both $\pre_{\enc}(z)$ and $g(z)$ are connected sets.
\end{lemma}
%
The proof for the connectivity of $\pre_{\enc}(z)$ relies on the fact that full rank affine transformations and strictly monotonic activations each have connected pre-images. Then by continuity of both $f$ and $\enc$, set $g(z)$ must be connected. 
Based on Lemma~\ref{lemma: connected pre and post image}, the CICO property of $\enc$ ensures that 
$g$ 
always outputs a connected set, enabling us to learn $g$ in a 
non-conservative 
manner via GPR.

\vspace{-2mm}
\subsection{Learning a Latent Space} \label{sec:autoencoder}
To learn the
latent space, we define three separate neural networks: (i) the encoder $\enc$, (ii) a latent space dynamics network $\hat{g}$
which models inclusion dynamics during training as a ball $\hat{g}(z) = \mathcal{B}_{\hat{g}_r(z)}(\hat{g}_c(z))$, and (iii) the decoder $\dec$. 
When training the networks, we propose the following loss function that consists of several objective functions to identify a latent space where an informative abstraction can be constructed, training all the networks simultaneously, i.e., 
$\enc, \hat{g}, \dec = \argmin_D\sum_{j=1}^4 \alpha_j L_j(\{x, x'\}) + \alpha_5 L_5(x, x'),$
where
$L_1(x)= \max\left\{0, \|\enc(f(x)) - \hat{g}_c(\enc(x))\| - \hat{g}_r(\enc(x))\right\},$ 
\quad
$L_2(x)=  \hat{g}_r(\enc(x)),$ 
\quad
$L_3(x)= \|x - \dec(\enc(x))\|,$ 
\quad
$L_4(x)= \expect_{q(z\mid x)}[-\log p(x \mid z)] + \text{KL}(q(z \mid x) \| p(z)),$ 
\quad 
$L_5(x,x')= \bigl|\|x - x'\|_{A_X} - \|\enc(x) - \enc(x')\|_{A_Z}\bigr|$ 
and $\alpha_j > 0$ are weights for each objective.  
$L_1$ and $L_2$ encourage that the evolution of the dynamics when mapped to the latent space should be captured by a small ball predicted by $\hat{g}$.
$L_3$ is a standard loss that attempts to learn an architecture where the composition $\dec(\enc(\cdot))$ is Identity.
Loss functions $L_1$-$L_3$ encourage
a latent space where dynamics do not grow rapidly, which helps construct an abstraction with sparse transitions, and requires the latent space to remain informative enough to decode accurately.  
Objectives $L_4$ and $L_5$ are introduced to further encourage the expressivity of the latent space. 
$L_4$, a Variational Autoencoder loss term, acts as a regularizer that prevents the collapse of the latent space; post training the encoder is used as a deterministic function. 
Variation in training forces the encoder to distinguish between similar inputs and prevents other loss terms from being trivially satisfied by mapping all inputs to a point.
$L_5$ 
encourages the notion of distance to be similar between the original space and the latent space, with weighting matrices $A_X, A_Z$,
and encourages regions to be separated in the latent domain. 
It should be noted, that while $\hat{g}$ and $\dec$ are required for the encoder $\enc$ to be well informed during training, they hold no guarantees on correctness and cannot be directly used to create the abstraction. We also note that $\dec$ can be defined as a set-valued neural network, to more accurately represent how points in the latent space are mapped to sets in the high-dimensional space. 

\subsection{Nondeterministic Transition System 
} \label{sec:NTS}
In the latent space, we construct an NTS abstraction, as defined below. 
\begin{definition}
    [NTS] A Transition System (TS) is a tuple $\mathcal{N} = 
    (Q, T, Q_G, Q_U)$, 
    where 
    $Q$ is a finite set of states, 
    $T: Q \times Q \rightarrow \{0, 1\}$ is a transition function such that $T(q ,q') = 1$ if $q' \in Q$ is a successor of $q \in Q$,
    and
    $Q_G,Q_U \subseteq Q$ are the set of goal and unsafe states, respectively.  
    $\mathcal{N}$ is called \emph{non-deterministic} (NTS) if there exists a $q \in Q$ 
    such that $\sum_{q'\in Q} T(q,q') > 1$.
\end{definition}

A path on NTS $\mathcal{N}$ starting at $q \in Q$ is a sequence $\omega_{q} = \omega_{q}(0)\omega_{q}(1)\ldots$ s.t. $\omega_{q}(0) = q$ and for all $i \geq 0$, $T(\omega_{q}(i), \omega_{q}(i+1)) = 1$. 
Note that, from $q$, there may exist multiple paths due to non-determinism in $T$. The set of all paths from $q$ is denoted as $\Omega_q$. 
The reach-avoid property over NTS $\mathcal{N}$ is denoted by $\varphi \equiv (Q_G,Q_U)$. We say path $\omega_q$ satisfies $\varphi$, i.e., $\omega_q \models \varphi$, if it reaches $Q_G$ while avoiding $Q_U$, i.e., $\exists i \in \nats$ s.t. $\omega_q(i) \in Q_G$ and $\forall j < i, \> \omega_q(j) \notin Q_U$.
Then we say initial state $q \models \varphi$ iff $\omega_q \models \varphi$ for all $\omega_q \in \Omega_q$.

Below, we describe a construction of an NTS abstraction of System~\eqref{true_dynamics} in the latent space such that if abstraction state $q \models \varphi$, then it is guaranteed that for every $x \in \pre_{\enc}(q)$, the original system trajectory $\omega_x \models \varphi$. 
This requires sound mapping of $X$, $X_U$, and $X_G$ to the latent space as well as careful construction of $T$ of the NTS.

\textbf{Mapping Regions of Interest. \ } \label{subsec:mapping}
To reason about $\varphi$ in the latent space,
it is necessary to map the sets $X, X_U, X_G$ from $\reals^{n_x}$ to $\reals^{n_p}$. 
With $\enc$ as a CICO network, we guarantee that the boundary of a mapped convex set must come from the boundary of the input set. 
This notion allows us to use efficient sampling based methods to map a high-dimensional convex region down with a guaranteed under- and over-approximation of the region in the latent space, while avoiding the computational complexity of standard neural network verification methods.  

Given convex set $r \subset \reals^{n_x}$, 
let $r_{Z} = \post_{\enc}(r)$ be the encoding of this set.
Directly calculating $r_{Z}$ is generally infeasible; however, we can sample a set of $N$ points from $r$ and map each point through $\enc$ to define a set in $\reals^{n_p}$ and make use of the convex hull of this set to over-approximate $r_{Z}$
with high confidence through \textit{$\epsilon$-RandUp} \citep{lew2022simple}, shown in Proposition~\ref{prop:randup} in Appendix~\ref{app:propositions}, where the true posterior, with high confidence, is contained in the convex hull of samples that is expanded according to the desired confidence and number of samples. 
Specifically,
Proposition~\ref{prop:randup} states that 
$\mathbb{P}\big( r_{Z} \subset (\conv(D_Z) \oplus \mathcal{B}_{\epsilon(N, \delta)}(0)) \big) \geq 1 - \delta$, where $D_Z = \{\enc(x_j) \mid x_j \sim r\}_{j=1}^N$ is a set of encoded points sampled from $r$, $\epsilon(N, \delta)$ is an expansion term based on the number of samples $N$ and confidence $\delta$, and $\oplus$ is the Minkowski sum. 
A visualization of the process is shown in Figure~\ref{fig:ICNN_embedding}. 



We note that alternate methods, e.g., \cite{wang2021beta}, 
can be used for a deterministic, instead of confidence-based, over-approximation of a region, but at the cost of much reduced efficiency.

\textbf{NTS states $Q$, $Q_G$, and $Q_U$. \ } 
To define NTS states, 
we first use $\epsilon$-RandUp to construct the operator 
$$\mathfrak{R}(r, N, \epsilon) := \conv(\{\enc(x_j) \mid x_j \sim r\}_{j=1}^N) \oplus \mathcal{B}_{\epsilon}(0).$$
Note that, with $\epsilon=\epsilon(N,\delta)$, $\mathfrak{R}$ 
returns 
a high-confidence over-approximation of $\post_{\enc}(r)$.
We next construct a latent domain $Z$ such that only states in domain $X$ can map to $Z$.
%
To this end, we assume access to the reachable domain $\bar{X} = \post_f(X) \cup X$, a common assumption \citep{KAZEMI2024101467, rungger2016scots, gracia2025beyond}, and define $X^c = \bar{X} \setminus X$.
%
%
%
%
%
We first compute 
candidate latent domain from samples of $X$ as
$\check{Z} = \mathfrak{R}(X, N, 0)$. We then partition $X^c$ into a set of convex regions $\{X^{c,k}\}_{k=1}^K$
and, 
for each region, compute a high-confidence over-approximation in the latent space, i.e., $\hat{z}_{X^{c},k} := \mathfrak{R}(X^{c,k}, N, \epsilon(N, \delta/K))$. 
Finally, we remove these regions from the candidate latent domain and
define the latent domain as $Z = \check{Z} \setminus \cup_k \hat{z}_{X^{c},k}$. 
By construction, this approach guarantees that, with probability at least $1-\delta$, $(\pre_{\enc}(Z) \cap \bar{X}) \subseteq X$. 

We next construct latent representations of $X_U$ and $X_G$.
We over-approximate the posterior of each unsafe region $u_i \in X_U$ under $\enc$ in the latent space using $\epsilon$-RandUP as $\hat{z}_{u,i} := \mathfrak{R}(u_i, N, \epsilon(N, \delta/l))$ for $i \in [0, l]$. This yields a sound (conservative) representation of the unsafe regions, with a joint confidence $1 - \delta$ of over-approximating $X_U$. 
For $X_G$, 
we instead construct a latent region $z_G$ such that only states in $X_G$ map to it,
using the same procedure employed to construct the latent domain $Z$.
Specifically, with confidence at least $1-\delta$,
$\pre_{\enc}(z_G)\cap\bar X\subseteq X_G.$
%
%
Intuitively, $z_G$ contains latent states that, within the reachable domain $\bar X$, can only correspond to states in $X_G$.

Based on these regions, we define abstraction states.
With an abuse of notation, let $q$ denote both the discrete state $q \in Q$ of the NTS and the region $q \subseteq \reals^{n_p}$ that it occupies. 
We partition $Z$ into a set of regions $\overline{Q}=\{q_0, \ldots, q_{|\overline{Q}|}\}$ that respect $R = \{\hat{z}_{u,i}\}_{i=1}^l \cup \{z_G\}$, 
i.e., for every $r \in R$ and $q \in Q$ either $q \subseteq r$ or $q \cap r = \emptyset$. 
Denote by $q_u = \reals^p \setminus Z$ the set of points outside latent domain $Z$. Then, we define NTS abstraction states as $Q = \overline{Q} \cup \{q_u\}$ with $Q_G = \{q \in Q \mid q \subseteq z_G\}$ and $Q_U = \{q \in Q \mid q \subseteq \hat{z}_{u,i} \text{ for some } i\} \cup \{q_u\}$.  
The following lemma establishes the correctness of this construction.

\begin{lemma}
    [Sound Reach-Avoid Abstraction States] 
    \label{lemma: correct state id}
    Consider latent (domain) region $\bar{Z} = \cup_{(Q \setminus \{q_u\})} q$, latent goal region $Z_G = \cup_{Q_G} q$, and latent unsafe set $Z_U = \cup_{(Q_U \setminus \{q_u\})}q$ constructed from $X$, $\bar{X}$, $X_G$, and $X_U$  per procedure described in Section~\ref{subsec:mapping}.
    Then, it holds that $(\pre_{\enc}(\bar{Z}) \cap \bar{X}) \subseteq X$.  
    Further, with confidence at least $1-\delta$, it holds that, for every $u \in X_U$,  $u \subseteq \pre_{\enc}(Z_U)$ and $(\pre_{\enc}(Z_G) \cap \bar{X}) \subseteq X_G$. 
\end{lemma}

\subsection{NTS Transitions: Modeling the Latent Inclusion Dynamics} \label{sec:GP_modeling}

To generate NTS transitions $T$, we 
propose a novel approach to model inclusion dynamics via GPR. 

\textbf{Inclusion GP. \ } 
Note that the latent space inclusion dynamics $g$ has at most an intrinsic dimension $n_p$, i.e., for a given input $z$ the output set fills at most an $n_p$ dimensional volume. Hence, the output can be sufficiently parameterized by an additional $\leq n_p$ dimensional variable. Effectively, in the latent space, we can lift $g$ to a higher dimension where dynamics follow difference equations as $\tilde{g}: \reals^{n_p} \times C \rightarrow \reals^{n_p}$ where $z^+ = \tilde{g}(z, c)$ for some $c \in C \subset \reals^{n_p}$, where $C$ is a compact set (e.g. a hypercube). With $\tilde{g}(z, c)$ a continuous function, GPR can be used to learn the $j$-th dimension of $\tilde{g}$. On the bounded domain $Z \times C \subset \reals^{n_p}$, the Sobolev space $H^s(Z, C)$ is an RKHS with $s > {n_p}/2$ that contains 
continuous functions; $\tilde{g}$ is continuous and therefore lives in this RKHS\footnote{A common choice for continuous dynamical systems is the squared exponential kernel; however, selecting an appropriate kernel for $\tilde{g}$ requires further investigation. See Appendix \ref{app:kernel} for detailed discussion.}. 

A key challenge is that the parameter $c$ is unobserved, i.e., the data are $(z_i, z_i^+)$ rather than $((z_i, c_i), z_i^+)$. 
Thus, to apply GPR, we must infer latent variables $c_i$ consistent with the observed pairs and then can represent inclusion dynamics by predicting over the entire set $C$. 
We first observe that the latent space dynamics can be described according to the original input space as $z^+ = \mathcal{T}(x) = \enc(f(x))$, and assert that there exists an equivalent mapping to $\mathcal{T}$ as $z^+ = \tilde{g}(\enc(x), \psi_c(x))$ where $\psi_c: \reals^{n_x} \rightarrow C$.
We then learn $c_i$ through a unique architecture for a deep kernel to model the dynamics $\mathcal{T}(x)$ with a kernel architecture $\kappa_X(x, x') = \kappa(\psi_{l}(\enc(x), \psi_c(x)), \psi_{l}(\enc(x'), \psi_c(x')))$ where $\psi_c, \psi_l$ are learned neural networks. This kernel has an equivalent representation in the latent space with deep kernel $\kappa_Z((z, c), (z', c')) = \kappa(\psi_{l}(z, c), \psi_{l}(z', c'))$ using dataset $\{(z_i, c_i, z_i^+)\}_{i=1}^m = \{\enc(x_i), \psi_c(x_i), \enc(f(x_i))\}_{i=1}^m$. Then, following the approach in~\cite{reed2023promises}, the evolution of the $j$-th output dimension of $\tilde{g}$ is modeled with $\mu^{(j)}_D(z, c)$ on $\kappa_Z$. Under the assumption that $\mathcal{T}$ is in the RKHS of $\kappa_X$, then $\tilde{g}$ is in the RKHS of $\kappa_Z$. Since there are infinitely many equivalent representations of the latent space dynamics under the parameter $c$, any accurately learned representation through $\kappa_X$ sufficiently represents the latent space dynamics. 
We assume sufficient data to learn an accurate kernel representation through neural networks $\psi_c$ and $\psi_l$.

We refer to this GP model as an \textit{Inclusion GP}. A visualization of the Inclusion GP prediction is shown in Figure \ref{fig:Inclusion_GP} in Appendix~\ref{app:IGPR}.
Intuitively, an Inclusion GP captures the evolution of a non-deterministic system by predicting in a lifted domain with an unobserved coordinate. We then bound the error of the set prediction using RKHS analysis as follows.

\begin{theorem}[Inclusion GP Error Bounds]
    \label{thm: transition function construction}
    Let $g, \tilde{g}$ be the latent dynamics, 
    constants $B_j > 0$ and $d_j^* \geq 0$ hold for each $\tilde{g}^{(j)}$ according to Proposition \ref{our_det_theorem},
    and $\epsilon^{(j)}(z, c) = \sigma_\D^{(j)}(z,c)\sqrt{B_j^2 - d_j^*}$. Then,  
    \begin{align}
        g^{(j)}(z) \subseteq [\check{\mu}^{(j)}_{D}(z) - \overline{\epsilon}^{(j)}(z),  \; \hat{\mu}^{(j)}_{D}(z) + \overline{\epsilon}^{(j)}(z)] ,
    \end{align}
    where 
    $\check{\mu}^{(j)}_{D}(z)$ and $\hat{\mu}^{(j)}_{D}(z)$ are the min and max of $\mu^{(j)}_{D}(z, c)$ over $c$, 
    and $\overline{\epsilon}^{(j)}(z) = \max_c \epsilon^{(j)}(z, c)$.
\end{theorem}
The proof relies on the fact $g$ and $\tilde{g}$ are functionally equivalent representations and that $\overline{\epsilon}^{(j)}(z)$ is a error bound for each $\tilde{g}^{(j)}(z, c)$ prediction, and a union of these error sets would then contain $g(z)$.

We note that while our formulation defines $g(z)$ as a connected set, this is not a requirement to perform Inclusion GPR. If the output is not connected, the same approach can be used, but the GP prediction becomes a conservative connected set representation of the disjoint evolution. 

\textbf{Latent RKHS Constants. }
As stated in~\cite{reed2024error}, $d_j^*$ 
(the parameter used in the GPR error bound in Proposition~\ref{our_det_theorem} and Theorem~\ref{thm: transition function construction})
can be conservatively set to 0 to avoid difficulty in calculating a nontrivial value, as in our case  
$d_j^* \leq \tilde{g}^{(j)}(Z)^T(K_{(Z,c), (Z,c)} + \sigma_n^2 I)^{-1} \tilde{g}^{(j)}(Z)$ would require being able to bound the effect of $\tilde{g}(Z, c)$ on the quadratic form. 
$B$ can be formally bounded using approaches in~\cite{jackson2021formal}, and only depends on having a Lipschitz constant for the dynamics. The Lipschitz constant for $g$ can be bounded using the constants for $\enc$ and the dynamics in the original space, or bounded using sampling based approaches as in~\cite{wood1996estimation}.

\textbf{NTS Transition Function. }
Using an Inclusion GP, we can propagate region $q \subset \reals^{n_p}$ for a subset $v \subset [a, b]$ as
$\mathrm{Im}(q, v) = \{w \mid w^{(j)} = \mu^{(j)}_D(z, c),\> c \in v,\> z \in q\}.$
Let $\epsilon^{(j)}_{q,v} = \sup_{z \in q, c \in v}\epsilon^{(j)}(z, c),$
and the worst case 
regression error over a region $q$ be 
$e^{(j)}(q, v) = \sup_{z \in q, c \in v} |\mu_{D}^{(j)}(z, c) - \tilde{g}^{(j)}(z, c)|.$
Then, $e^{(j)}(q, v) \leq \epsilon^{(j)}_{q, v}.$ 
Using operator $\text{Im(q,v)}$ and error bound $\epsilon_{q, v}$, Proposition~\ref{prop:trans} shows how to over-approximate the transition from $z \in q$ to $q'$ under dynamics $g$, defining transition function $T$.

\begin{proposition}[\cite{jackson2021formal}, Theorem 1] 
    \label{prop:trans}
    Let $\bar{q}'(\epsilon)$ be a region that is obtained by expanding each dimension of $q'$ with the corresponding scalars in $\epsilon_{q, u}$. Also let $V$ be a partition of $C$ such that $\cup_{V} v = C$. Define the transition function of the NTS $\mathcal{N}$ abstraction as
    $T(q, q') = \max_{v \in V} \big( \ind_{\bar{q}'(\epsilon)}(\mathrm{Im}(q, v))\big),$
    where $\ind_W(H)$ is the indicator function that returns $1$ if $W \cap H \neq \emptyset$ and $0$ otherwise. 
    Then, it holds that
    $\forall z \in q, \> g(z) \subseteq \{z' \in q' \mid q' \in Q,\; T(q,q')=1\}$.
\end{proposition}

We finalize the construction of $T$ by making $q_u$ an absorbing state, i.e., $ T(q_u, q') = 1$ if $q'=q_u$ and $0$ otherwise. $T$ then conservatively captures the behavior of the inclusion dynamics under discretization, i.e., for $z' \in q'$ and $z \in q$ if $z' \in g(z)$ then $T(q, q') = 1$.
With $T$ defined, the NTS abstraction is completed. 
The following theorem shows our abstraction $\mathcal{N}$ is sound w.r.t. reach-avoid property $\varphi$.

\begin{theorem}[Correctness]
    \label{thm: soundness}
    Let NTS $\mathcal{N}$ be constructed with the states 
    defined as in Section~\ref{sec:NTS} and transition function defined in Proposition~\ref{prop:trans}. 
    If a state $q \in Q$ satisfies $\varphi$ on $\mathcal{N}$, then for every $x \in \pre_{\enc}(q) \cap X$, the trajectory $\omega_{x}$ satisfies $\varphi$
    with confidence at least $1-\delta$, i.e.,
    with confidence at least 
    $1-\delta$, 
    \quad $q \models \varphi \implies \omega_x \models \varphi \quad$
    $\forall x \in \pre_{\enc}(q) \cap X.$
\end{theorem}

Proposition~\ref{prop:trans} establishes that 
for every trajectory $\omega_{x_0}$ s.t. $\enc(x_0) \in q_0$, there exists a trajectory $\omega_{q_0} \in \Omega_{q_0}$ such that $\enc(\omega_{x_0}(k)) \in \omega_{q_0}(k)$ for all $k \geq 0$,
and Lemma~\ref{lemma: correct state id} asserts $Q_U$, $Q_G$ conservatively represent $X_U$, $X_G$ with confidence $1 - \delta$.
Then it can be seen that if $\omega_{x_0} \not\models \varphi$ there must exist a path $\omega_{q_0}$ where $\enc(\omega_{x_0}(k)) \in \omega_{q_0}(k)$
for all $k \geq 0$ s.t. $\omega_{q_0} \not\models \varphi$, which establishes the contrapositive of  
$q \models \varphi \> \implies \> \omega_x \models \varphi$ for all $x \in \pre_{\enc}(q).$

\textbf{Verification and Decoding. }
Given that the NTS $\mathcal{N}$ is a sound abstraction of System~\eqref{true_dynamics} per Theorem~\ref{thm: soundness} in the latent space, we use existing methods to verify $\mathcal{N}$ 
and identify the set of states that are guaranteed to satisfy 
$\varphi$~\citep{baier2008principles}. 
Due to the uncertainty introduced by non-determinism and discretization in the latent space, we obtain three sets from verification: $Q_{yes}$ (set of states guaranteed to satisfy $\varphi$,
$Q_{no}$ 
(set of states that can never satisfy $\varphi$),
and $Q_? = Q \setminus (Q_{yes} \cup Q_{no})$. 
We can then make use of convex optimization to decode (as $\enc$ is convex) latent states to identify satisfying states in the original space (Appendix~\ref{app:inverse}).

    Finally, we remark that the proposed latent-space NTS abstraction $\mathcal{N}$ can be straightforwardly extended to verify LTL properties of System~\eqref{true_dynamics} (see Appendix~\ref{sec:app_LTL}).

\section{Case Studies}
    \label{sec:examples}
    
We evaluate our approach through several case studies.
We first demonstrate the effectiveness of the encoder training. We then present three illustrative examples:
(i) dimensionality reduction beyond the minimal realization while preserving correct satisfaction reasoning,
(ii) correct reasoning with over-parameterized models, and
(iii) a high-dimensional system where traditional abstraction methods suffer from state explosion. 
Implementation details 
are provided in Appendix~\ref{app:experiment_params}. 
In all experiments, 
$\delta = 10^{-5}$, ensuring
abstraction confidence of at least $99.99\%$.

\textbf{Training loss. }
We first perform an ablation study on the loss terms used to train the encoder, which identifies the necessity of each term to learn an interpretable latent space suited for abstraction. In short, our results show that without $L_4$ and $L_5$ training is unstable and the latent space collapses and without $L_1$ and $L_2$, latent dynamics grow rapidly. See Appendix~\ref{app:ablation} for results and discussions.

Next, we show the entire framework on multiple case studies.  Figures~\ref{fig:3D_nonlinear}-\ref{fig:lidar}
show verification results in the latent space, using the following legend: $Q_G$ (dark green), $Q_U$ (black), $Q_{\text{yes}}$ (light green), $Q_{?}$ (yellow), and $Q_{\text{no}}$ (red).
Benchmark statistics over 5 trials of each experiment are reported in Table~\ref{table:ver_results}, showing abstraction generation time is dominant, which is due to computation of $\mathrm{Im}(q,v)$.

\begin{figure}[t]
    \centering
    \begin{subfigure}[b]{0.24\linewidth}
        \includegraphics[width=\linewidth]{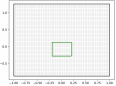}
        \caption{3D latent space}
    \end{subfigure}
    \begin{subfigure}[b]{0.24\linewidth}
        \includegraphics[width=\linewidth]{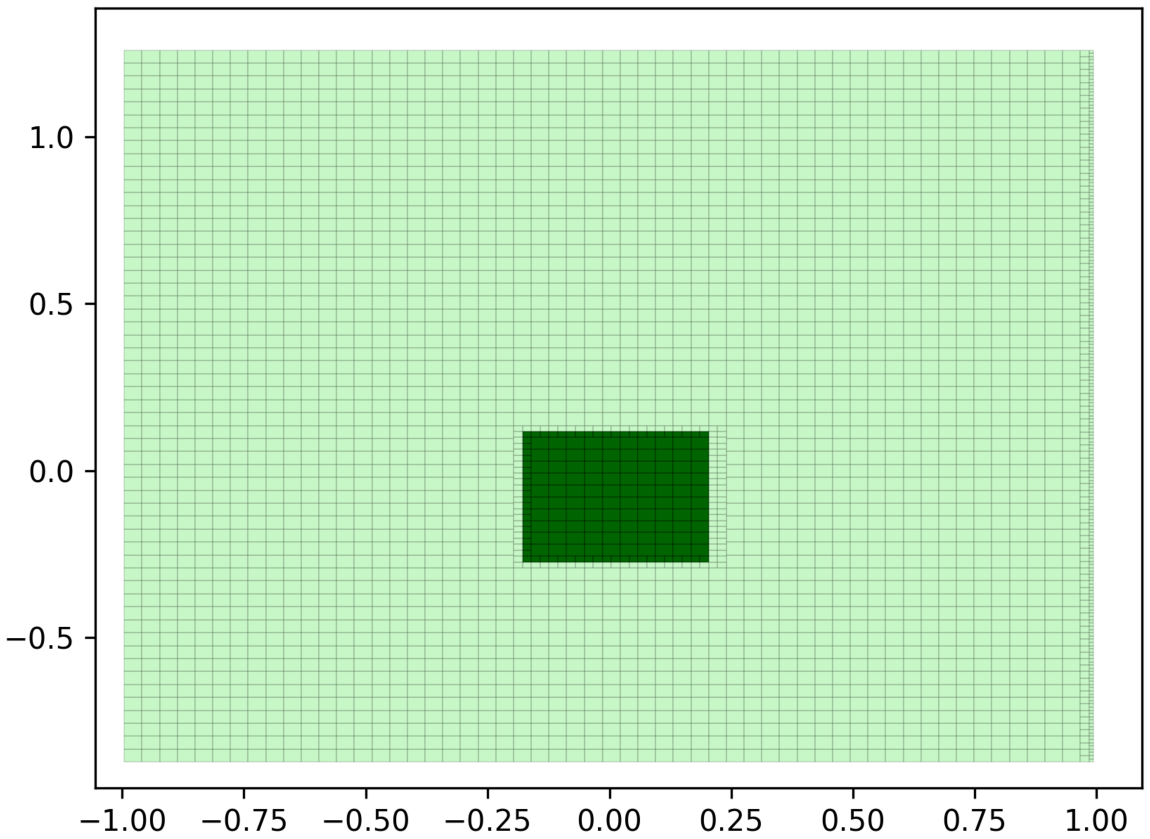}
        \caption{3D Verification}
    \end{subfigure}
    \begin{subfigure}[b]{0.24\linewidth}
        \includegraphics[width=\linewidth]{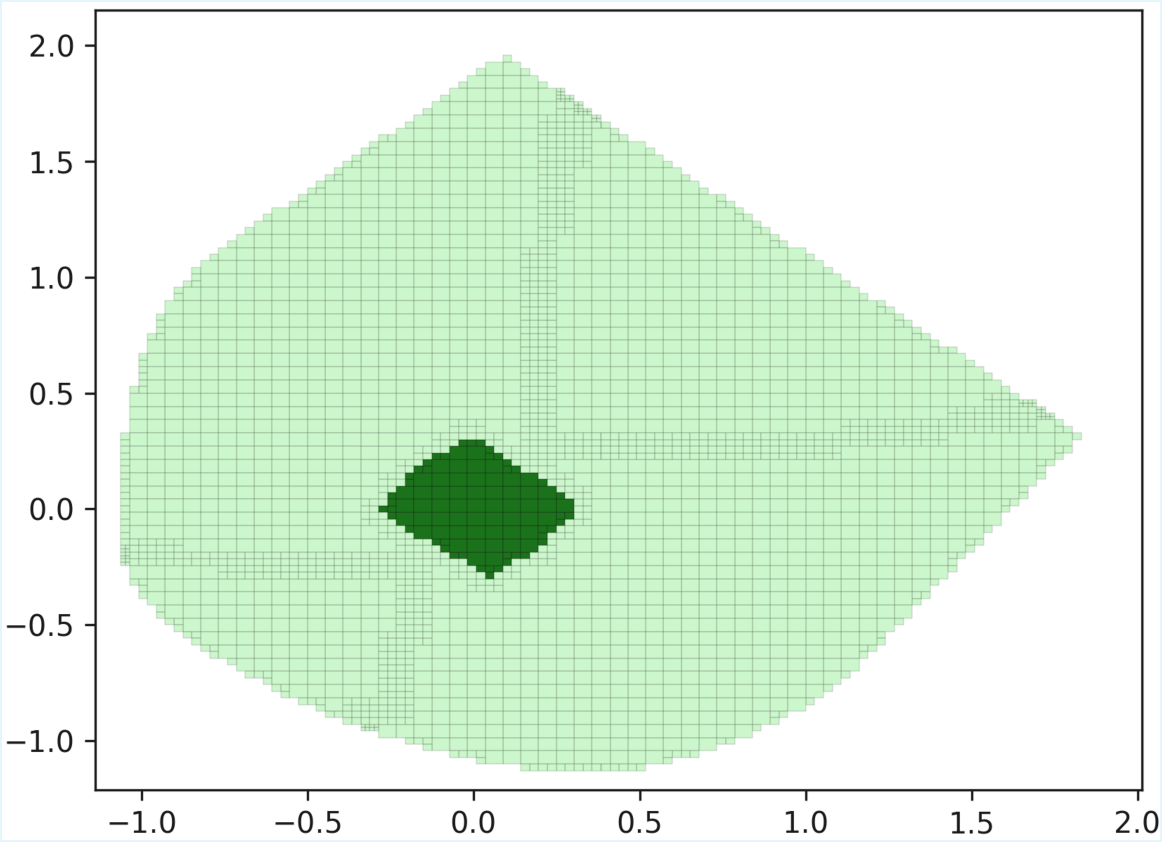}
        \caption{6D Verification}
    \end{subfigure}
    \begin{subfigure}[b]{0.24\linewidth}
        \includegraphics[width=\linewidth]{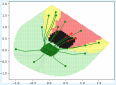}
        \caption{6D with Obstacle}
    \end{subfigure}
    \vspace{-1mm}
    \caption{Results for ``\textit{3D-Nonlinear-to-2D}'' and ``\textit{6D-to-2D}'' case studies. (a) The 2D latent space of the 3D system with the goal shown outlined in green, the domain outlined in black, and verification results in (b).
    (c) Verification results for the ``\textit{6D-to-2D}'' case study without obstacles, and (d) with an obstacle. (d) Encoded sample trajectories: 
    green satisfy and red violate $\varphi$ in the original space.
    }
    \label{fig:3D_nonlinear}
    \vspace{-2mm}
\end{figure}

\begin{figure*}[t]
    \centering
    \begin{subfigure}[t]{0.28\linewidth}
        \centering
        \includegraphics[width=\linewidth]{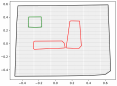}
        \caption{Initial latent discretization}
        \label{fig: 26D latent space}
    \end{subfigure}
    \begin{subfigure}[t]{0.28\linewidth}
        \centering
        \includegraphics[width=\linewidth]{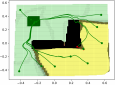}
        \caption{Latent verification}
        \label{fig: 26D verification}
    \end{subfigure}
    ~
    \begin{subfigure}[t]{0.205\linewidth}
        \centering
        \includegraphics[width=\linewidth]{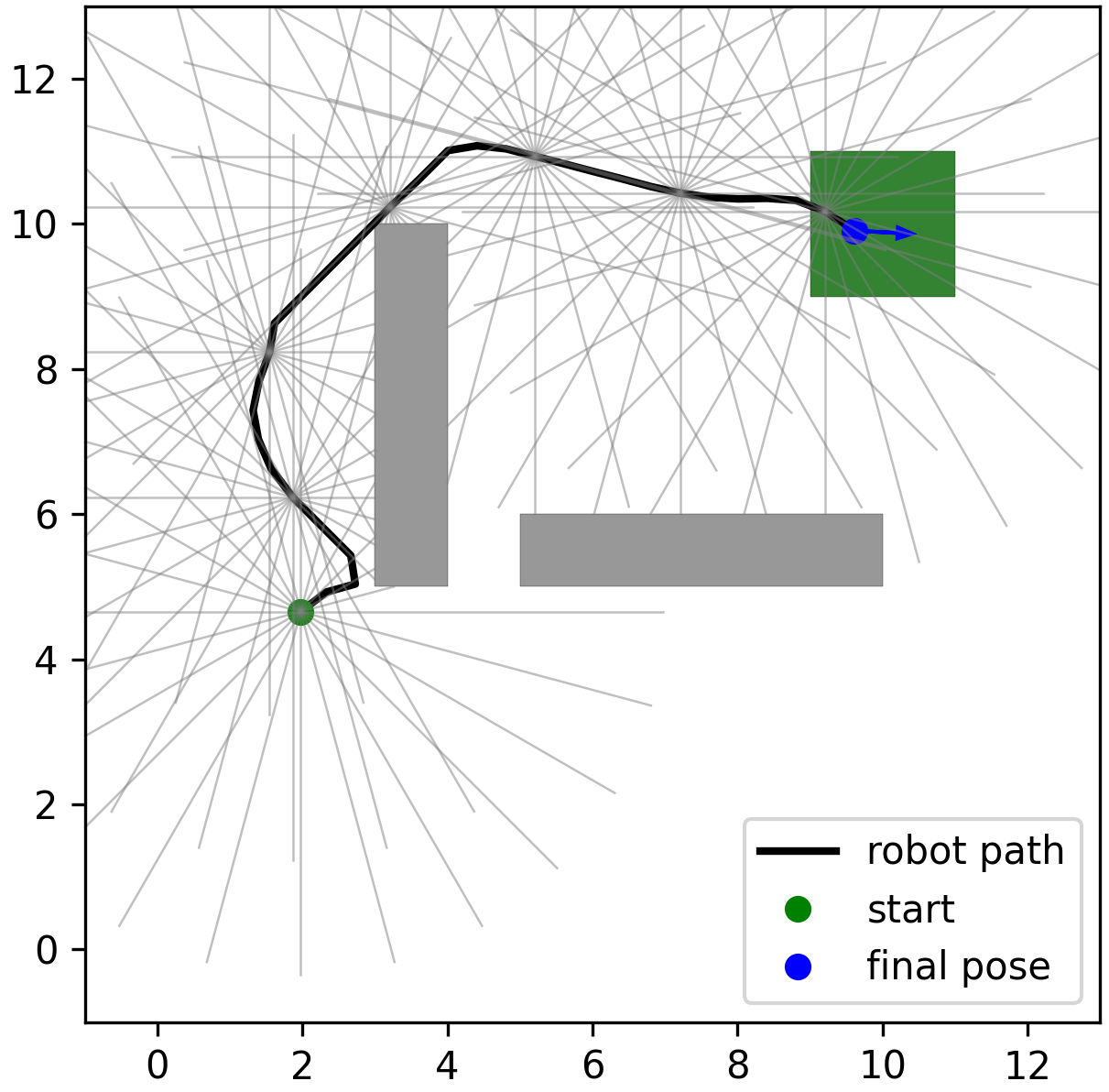}
        \caption{Safe Traj.}
        \label{fig: 26D safe traj}
    \end{subfigure}
    \begin{subfigure}[t]{0.205\linewidth}
        \centering
        \includegraphics[width=\linewidth]{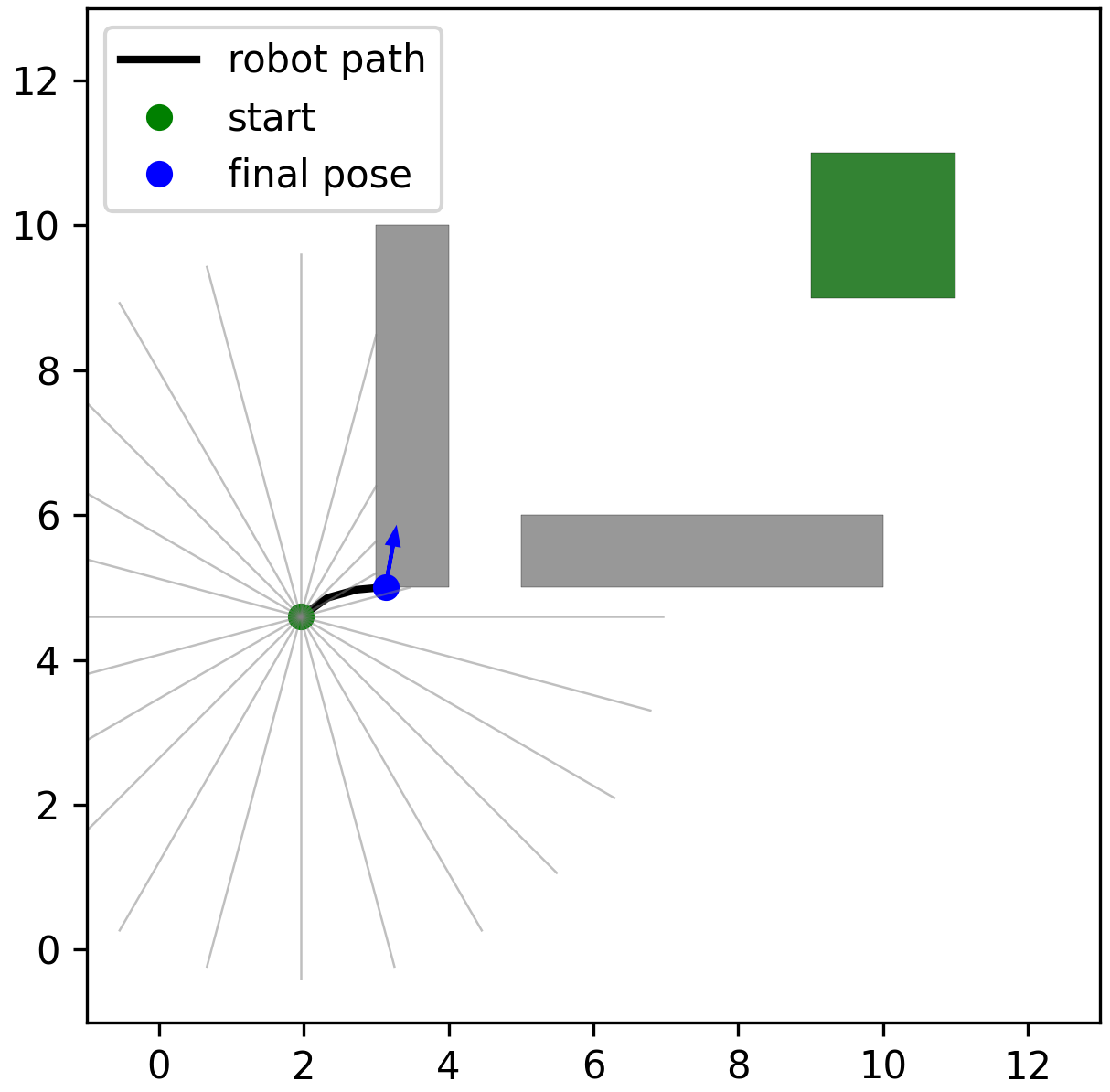}
        \caption{Unsafe Traj.}
    \end{subfigure}

    \vspace{-1mm}
    \caption{
    Results for the ``\textit{26D-to-2D}'' case study. (a) The 2D latent space and regions. (b) Verification results with (encoded) sample trajectories from the original space that satisfy $\varphi$ 
    in green and a violating trajectory in red. (c)-(d) Safe and Unsafe trajectories that begin in the same region $q$ from (b) shown in the original space of the system (LiDAR shown every 5 steps).
    }
    \label{fig:lidar}    
    \vspace{-1mm}
\end{figure*}


\textbf{3D Nonlinear to 2D. }
We consider a problem where all dimensions are relevant to the dynamics but we would still like to reduce the dimensionality to make abstraction more tractable. 
We also consider if we can make use of knowledge about the system to improve verification results and require that the encoder is learned such that the mapping for $x^{(1)}$ is identity. We emphasize that the 3D dynamics are the minimal realization for the system, and traditional techniques like PCA cannot identify a lower dimensional space without removing information. This system follows decaying rotational dynamics, stabilizing at the origin. In this example, 
$X = \{x \mid \|x\|_\infty \leq 1\}$,
$X_G =\{x \mid \|x\|_\infty \leq 0.2\}$, and $X_U = \emptyset$. 
The results for verification are shown in Figure~\ref{fig:3D_nonlinear} a-b. We easily identify that the entire latent space satisfies our task.
%
%
To demonstrate the efficacy of our approach, we also perform verification in the original $3$D space. With a data-driven abstraction (unknown dynamics), uncertainty in the slow-moving dimensions introduces worst-case self-transitions, preventing verification of eventual reachability to the origin despite the existence of valid paths. In contrast, a model-based abstraction (known dynamics) verifies reachability from $100\%$ of the space. The latent abstraction also successfully verifies the system by entangling slow- and fast-moving states, avoiding self-transitions.

\textbf{6D Nonlinear to 2D. }
Here, we show that our approach can reason about over-parameterized models as well. We expand the 3D dynamics of the previous example with nonlinear observations of the state to 6D. 
We use same definition of $X, X_U$, but $X_G = \{x \mid \| [x^{(1)},$ $x^{(2)}, x^{(3)}]\|_{\infty} \leq 0.2\}$. 
Verification results are shown in Figure \ref{fig:3D_nonlinear}c. We identify that the entire latent space satisfies $\varphi$.  
We then consider a slightly more complex scenario, and introduce an obstacle to the environment defined as $X_U := \{x \mid x^{(1)} \in [0.3, 0.6], x^{(2)} \in [-0.7, -0.3], x^{(3)} \in [-0.25, 0.25]\}$.
Verification results are shown in Figure~\ref{fig:3D_nonlinear}d, with almost half of space guaranteed to satisfy $\varphi$. We note that since our method over-approximates $X_U$, we are conservative in identifying $Q_{no}$, which can be seen from trajectories that pass through the obstacle in the latent space but are safe in the original space. 
%
%
Abstraction in the original $6$D space suffers from state explosion and over-parameterization. Even with a coarse discretization containing more abstract states than the latent approach, after $3063.62$ seconds (more time than the entire latent space approach) it verifies goal reachability only for states already in the goal, although it correctly verifies domain invariance. Like the $3D$ example, the slow moving states also have a stronger impact on the verification in the original space.


\textbf{26D to 2D. }
We demonstrate that our approach can be effective for systems that are high dimensional due to the control acting on high dimensional observations like LiDAR. We consider a robot in Figure~\ref{fig: 26D safe traj} that must navigate around obstacles and reach a goal location where the controller is only given LiDAR-like observations and the direction to the goal position. For simplicity of presentation we consider 2D single integrator dynamics and 24 range readings at evenly spaced headings with a maximum range of 5 units. The controller is learned using PPO in a noisy environment.  
We theorize that, given a set of obstacles and a goal position, a 2D latent space could be learned that models this system with minimal non-determinism. 
Figure \ref{fig: 26D latent space} shows the 2D latent space. We note that this space is visually similar to a flipped and scaled version of the original domain, demonstrating interpretability, but emphasize that LiDAR observations strongly impact the point in the latent space. 
Verification results are shown in Figure~\ref{fig: 26D verification}, verifying almost half of the latent space satisfies the specification. Figures~\ref{fig: 26D safe traj}-d show safe and unsafe original system trajectories that start from the same latent region near the obstacle in Figure~\ref{fig: 26D verification}.


\begin{table} [t]
    \caption{Timings for the encoder training ($t_{\enc}$), Inclusion GP training ($t_\text{IGP}$), abstraction generation ($t_\text{Abs}$), and the percent volume of space in the latent domain that is in $Q_{yes}$ (\%$\textrm{Vol}(Q_{yes}$)).}
    \label{table:ver_results}
    \vspace{-1mm}
    \centering
    \scalebox{0.8}{
    \begin{tabular} {l | c c c c}
            \toprule
            \textbf{System} & $\mathbf{t_{\enc}}$ (s) & $\mathbf{t_\text{IGP}}$ (s) & $\mathbf{t_\text{Abs}}$ (s) & $\mathbf{\%\textrm{Vol}(Q_{yes})}$ \\
            \midrule
            $3D$ & $125.75 \pm 2.66$ & $36.45 \pm 3.26$ & $495.14 \pm 150.09$ & $99.999 \pm 0.002$\\
            \midrule
            $6D$ & $434.34 \pm 17.88$ & $39.99 \pm 5.93$ & $1816.88 \pm 731.02$ & $100 \pm 0$ \\
            $6D$ Obs & $434.34 \pm 17.88$ & $39.99 \pm 5.93$ & $4227.05 \pm 4266.38$ &  $47.29 \pm 11.25$\\
            \midrule
            $26D$ & $503.00 \pm 7.32$ & $63.32 \pm 2.35$ & $6238.99 \pm 2137.75$ & $48.15 \pm 3.55$\\
            \bottomrule
    \end{tabular} 
    }
    \vspace{-1mm}
\end{table} 

    
\vspace{-1mm}
\section{Conclusion}
    \label{sec:conclusion}
    \vspace{-1mm}
    In this work we bridge a gap between dimensionality reduction and formal verification, defining a method that enables scalable abstraction-based verification for high-dimensional unknown systems.
We reason about non-determinism in the latent space dynamics through a novel application of GPR and 
show
that a finite abstraction can be correctly constructed in a latent space. 
While our approach is sound, there are multiple aspects that limit the flexibility, namely restricting the encoder to be a convex function. This restriction allows us to identify regions of interest in the latent space efficiently and correctly, but it also limits the expressivity of the latent domain. There also remain the open questions of how to accurately incorporate noise in the latent space and how to do  control synthesis.

\subsection*{AI use statement}
    In this work, we used generative AI tools for discovering related works, assisting in code generation, and improve grammar and sentence structure.
    We have not used generative AI tools for any other tasks with required or recommended disclosure. We have reviewed all AI-assisted work. LLM-generated code was verified and tested for correctness
    by 2 authors and all sources were checked for correctness. We take responsibility for the final content of this work,
    including text, claims or artifacts produced with the aid of generative AI.

\subsection*{Reproducibility statement}
    
    In our supplementary materials we provided a zip folder of the code used to generate the main results. Appendices contain detailed information on hyper parameter and experiment setup, as well as detailed proofs of theorem.

\bibliographystyle{iclr2027_conference}
\bibliography{references}


\appendix


\section{Figures}
    \label{sec:appendix}
    We note that these figures would ideally be in the main text, in the event this paper is accepted, the additional page limit given to the camera ready version would allow them to be moved into the main text.

\begin{figure}[h]
    \centering
    \includegraphics[width=0.8\columnwidth]{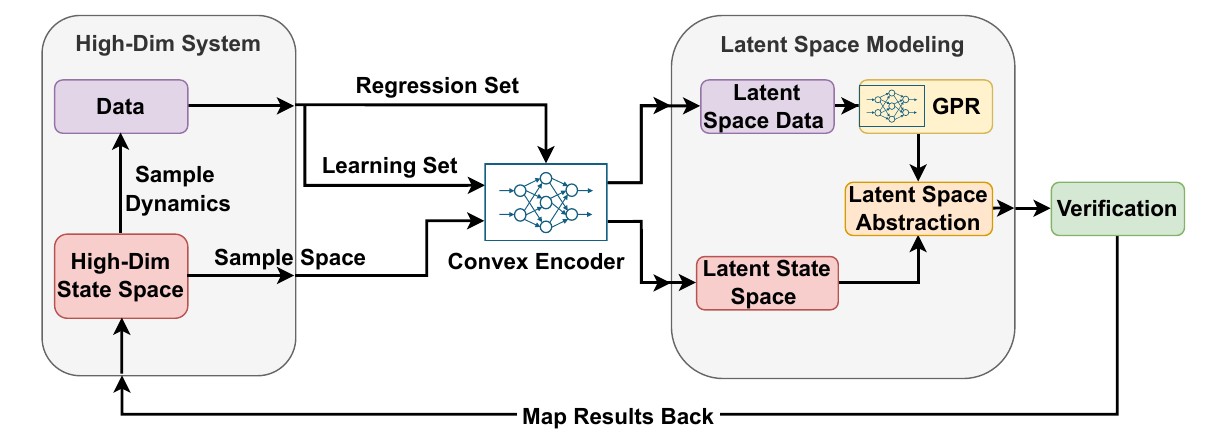}
    \caption{A flowchart of the method used in this paper. 
    }
    \label{fig:CALM_GPR}
\end{figure}

\begin{figure}[h]
    \centering
    \includegraphics[width=0.8\linewidth]{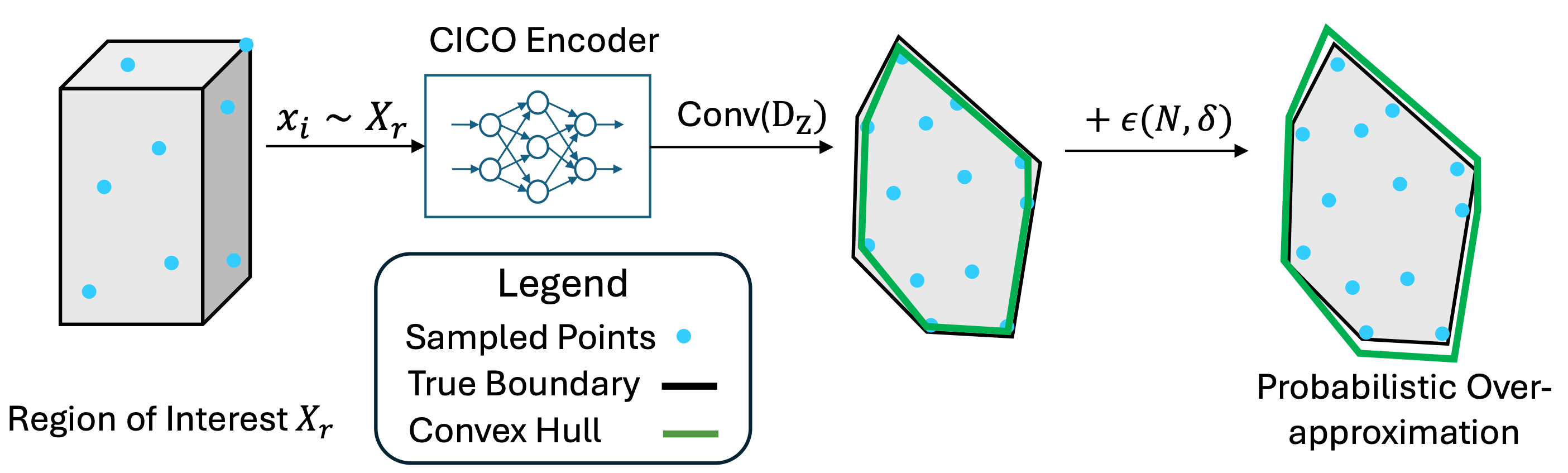}
    \caption{Visualization: mapping regions to the latent space.}
    \label{fig:ICNN_embedding}
\end{figure}

\section{Extended Discussion}
    \label{sec:appendix}


\subsection{Propositions}
\label{app:propositions}
Here we provide propositions that are used in our approach. These propositions are based on theorem from cited work.

The following proposition is on using samples and convex hulls to over-approximate the image of a set through a function.
\begin{proposition}
    [$\epsilon$-RandUp \citep{lew2022simple}, Theorem 2] Let $D_Z = \{\enc(x) \mid x \sim r \}_{i=1}^N$ be a set of encoded points sampled from region $r$
    and $r_{Z} = \post_{\enc}(r)$ be the true image of $r$ through $\enc$. Also, let $L_h$
    be the Lipschitz constant of $\enc$. Then, there exists an $\epsilon(N,\delta) \geq 0$ such that with probability at least $1 - \delta$,
    \begin{align}
        r_{Z} \subset \text{Conv}(D_Z) \oplus \mathcal{B}_{\epsilon(N, \delta)}(0), \label{eq:eps_randup}
    \end{align}
    where $\oplus$ denotes a Minkowsky sum, $\epsilon(N,\delta)$ satisfies $\delta = D(\partial r, \frac{\epsilon}{2L})(1 - \Lambda_{L_h}^\epsilon)^N$, $D(\partial r, \frac{\epsilon}{2L_h})$ is a covering number for boundary $\partial r$, and $\Lambda_{L_h}^\epsilon$ is a lower bound on sampling points within an $\frac{\epsilon}{2L_h}$ ball of each other.\label{prop:randup}
\end{proposition}
We can formally bound $\epsilon(N, \delta)$ with 
$L_h \left(\frac{v_{n_x}}{Nv_{n_p}}\log{\frac{(4L_h\sqrt{n_x})^{n_x}}{\delta}}\right)^{1/{n_p}}$ where $v_i$ is the volume of an $i$-dimensional ball, as under the assumption that $r$ is defined as a subset of a unit ball and is uniformly sampled we can take $D(\partial r, \epsilon/2L_h) \leq (4L_h\sqrt{n_x})^{n_x}$ and $\Lambda_{L_h}^\epsilon > \frac{v_{n_p}}{v_{n_x}L_h^{n_p}}\epsilon^{n_p}$ \citep{lew2022simple}.

\subsection{Remarks} \label{app:remarks}

\begin{remark} [Strictly Monotone Activations] \label{rem:monotone}
    Note that if the activation functions of $\enc$ are not strictly monotone, then the pre-image $\pre_{enc}(Z)$ of a connected set $Z \subset \reals^{n_p}$ can easily be disconnected, i.e., the encoder induces folding. Consider the ReLU activation; since the activations act dimension wise on the outputs of the affine transformation, ReLU can easily map disjoint regions of the input to the same value after multiple layers (e.g., two disjoint sets fall into the negative range of the pre-activation after a layer and then ReLU maps them both to 0). In contrast, strictly monotonic activation functions are injective and homeomorphic on their domains, which prevents folding.

    Here, we provide an explicit example that demonstrates that an arbitrary encoder can induce "folding". Consider the triangle function $T(z) = \text{ReLU}(z+1) - 2 \text{ReLU}(z) + \text{ReLU}(z-1)$, which is non-zero only on $[-1,1]$. We can define a $2$-layer network $f(x,y) = \text{ReLU}( T(y + 2) + T(x + 2) - 2) + \text{ReLU}(T(y-2) + T(x -2) - 2)$, which maps from $2$D to $1$D. This network produces two isolated peaks centered at $(-2, -2)$ and $(2, 2)$ and is $0$ elsewhere. Although these regions are disjoint in the input space, they map to the same scalar value, demonstrating folding. This is due to a combination of non–full-rank weight structure and non-injective activations.
\end{remark}


\subsection{Example Latent Inclusion Dynamics}
\label{app:example_inclusion}
Consider the system 
$$\begin{bmatrix} x^+ \\ y^+ \end{bmatrix} = \begin{bmatrix} 1 &1 \\ 0.1 &-1 \end{bmatrix}\begin{bmatrix} x \\ y \end{bmatrix}$$ 
and an encoder that induces a non-invertible change in variable, e.g., $\enc(\begin{bmatrix} x \\ y \end{bmatrix}) = x + y$ where the latent state $z = x + y$. To compute the evolution of the latent dynamics, we can write
$$z^+ = x^+ + y^+ = (x + y) + (0.1x - y) = 1.1x.$$
Hence, the evolution of $z$ is uniquely determined by $x$, but in the latent space, we do not have access to $x$. Therefore, $z^+$ is not a single-valued function of $z$; it is set valued:
$$z^+ \in \{1.1x \> | \> x \in \reals, \exists y \;\; s.t. \;\; x + y = z\}$$
Hence, in a latent space of lower dimension than the minimal realization, the system evolves according to inclusion dynamics. In a correct minimal realization, the dynamics evolve according to a difference equation. By denoting the evolution in the latent space with the singleton $\{f(\px)\}$, the system can still be described with inclusion dynamics.

\subsection{Discussion on Kernel Choice} \label{app:kernel}

The RKHS for Sobolov spaces can be modeled exactly with the Matern kernel. The specific kernel may be difficult to determine in general, as if $f$ is continuously differentiable (or even twice differentiable) the Matern kernel may be a poor choice and have conservative predictions. We make use of the squared exponential (SE) kernel in our experiments. While the RKHS for the SE kernel is highly restrictive, the kernel is a universal kernel. This means it is capable of modeling any continuous function arbitrarily well and for any continuous target function and any tolerance $\epsilon > 0$ there exists a function in the RKHS that approximates it to within $\epsilon$ (i.e. $|f_{target} - f_{in\> SE\> RKHS}| < \epsilon$) on a compact domain \cite[Chapter 4]{steinwart2008support}. Hence, the verification results are then effectively on a function that lives in the SE RKHS which we assume to be arbitrarily close (i.e. $|f_{target} - f_{in\> SE\> RKHS}| < 10^{-10}$) to the true function $f_{target} = \enc(f(\cdot))$ without loss of generality. This is a standard assumption.

\subsection{Inclusion GP Visual} \label{app:IGPR}
Here we provide a visual of predicting a non-deterministic set evolution using the inclusion GP model. Notice how the predicted output set from a single point spans the range of observed evolutions and accurately captures non-determinism.
Figure~\ref{fig:Inclusion_GP} shows the Inclusion GP prediction from one point $z^*$ in the latent space for our 6D to 2D example, where we generate the samples $\X'_{z^*}$ by decoding $z^*$ with convex optimization.
\begin{figure}[ht]
    \centering
    \includegraphics[width=0.5\linewidth]{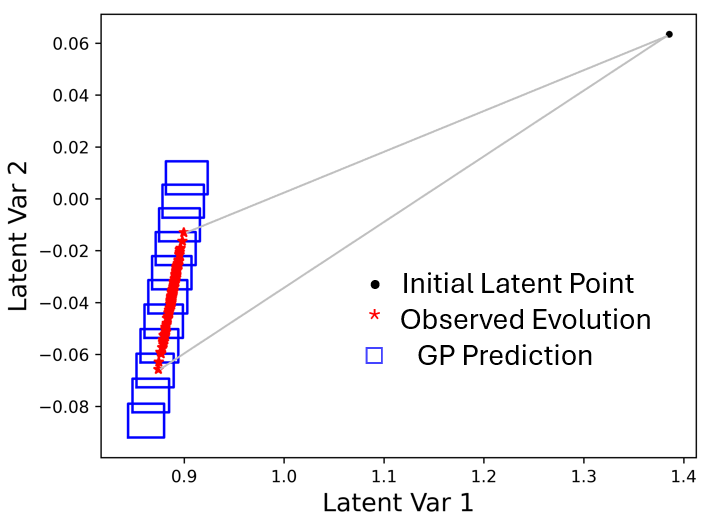}
    \caption{Visualization of Inclusion GP predictions. The initial latent variables $z^*$ is shown as a black dot (top-right), observed evolutions  $z_i^+$ as red stars (multiple samples of $\enc(\X'_{z^*})$), the predicted evolution for 10 $c^*$ values of the Inclusion GP with blue rectangles ($\mu^{(j)}_D(z^*, c^*) \pm \sigma_D^{(j)}(z^*, c^*)\sqrt{B^2 - d^*}$).}
    \label{fig:Inclusion_GP}
\end{figure}

\subsection{Decoding} 
\label{app:inverse}
Once the abstraction is verified in the latent space, it is often desirable to translate the results back to the high-dimensional system. We approach this by decoding the latent states in a correct manner.
Recall that the encoder is a convex function in our formulation. This enables decoding to be posed as a convex optimization problem, i.e., given a point $z \in Z$ in the latent space we can \emph{exactly} and efficiently decode it through the optimization of 
$$\argmin_{x \in X}\|\enc(x) - z\|.$$

While it may be difficult to find all points $x$ that map to a single $z$ we can relax the problem to instead decode a convex hyper-rectangle $q$ from our partition and identify a boundary from the original domain that encodes to this set. The search points for the optimization can be initialized with the decoder as $x_i = \dec(z_i)$ by sampling a set of points in the latent space 
$z_i \sim q$. The optimization can then be stopped when any $x_i$ is identified such that $\enc(x_i) \in q$. Let $q_c$ be the center point of $q$, $r_q$ the vector that defines the radius of the $q$ in each dimension, and denote $(\cdot)^{(j)}$ the $j$-th dimension of a vector $(\cdot)$, then this can be formulated with the constrained ($x_i \in X$) convex optimization.
\begin{align}
    \max_j \left(\frac{|\enc(x_i)^{(j)} - q_{c}^{(j)}|}{r_{q}^{(j)}} \right)< 1 \label{opt_decode}
\end{align}

Then we can use Eq. \eqref{opt_decode} to identify $\pre_{\enc}(q) \in X$.
The properties of our encoder enforce that the pre-image of a connected set is also connected, hence we can easily over-approximate the pre-image from samples using the same approach described in Section \ref{subsec:mapping}. Alternatively, we can make use of robust pre-image calculation methods for neural networks such as PREMAP \citep{zhang2025premap}, which provide a tight under-approximation of the pre-image as a union of polytopes.

\subsection{Experiment Parameters}
\label{app:experiment_params}

In each case study, we first learn $\enc$ from a dataset $D_{enc}$ (the Learning Set from Figure \ref{fig:CALM_GPR}) and then optimize the unobserved latent variable for the GP model with a separate dataset $D_{GP}$ (the Regression Set from Figure \ref{fig:CALM_GPR}). After generating the latent dataset $\{([z_i, c_i], z_i^+)\}_{i=1}^m$, we then learn a deep kernel model that predicts the evolution of the dynamics in the latent space with a network $\psi: \reals^{p+1} \rightarrow \reals^{p}$ and optimize the base kernel parameters for each dimension as $\kappa(\psi^{(j)}(z,c), \psi^{(j)}(z',c'))$, following methods outlined in \cite{reed2023promises}. Experiments were run on an Intel Core i7-12700K CPU at 3.60GHz with 32 GB of RAM.  

When training we first normalize the domain $X$ such that it is contained by the $n_x$-dimensional unit ball and use $500k$ samples to over-approximate regions with $\delta=0.0001$. When mapping regions to the latent space, we make use of kinematic relations between the states. This knowledge is often easier to identify than the dynamics and does not require us to know how the states evolve over time or which states define the minimal representation. This method helps reduce conservatism when over-approximating regions as our mapped points only contain kinematically feasible information. In each experiment $\bar{X} = X$.

Unless stated otherwise $\enc$ is a five layer network with size $\reals^{n_x} \rightarrow 1024 \rightarrow 512 \rightarrow 256 \rightarrow 128 \rightarrow \reals^{n_p}$ with a skip layer adding at the $512$ layer and the base kernel used in the deep kernel models is the squared exponential (SE) kernel with 1000 data points used for prediction. We set $\alpha_1 = 1$, $\alpha_2 = 0.7$, $\alpha_3 = \alpha_4 = 0.5$, and $\alpha_5 = 0.8$. For the 3D and 6D examples we set $A_X$ and $A_Z$ as identity weighting, and for the 26D example, we set $\alpha_5 = 5$ and remove weighting on the LiDAR inputs for the distance similarity.

As the code available for calculating the reachable set of a GP model relies on having axis-aligned hyper-cube inputs, our discretization in practice will make use of a non-uniform grid rather than a strict region-respecting discretization. We then further over-approximate the unsafe sets by defining $Q_U = \{q \in Q \mid \exists i \text{ s.t. } q \cap \hat{u}_{Z,i} \neq \emptyset \} \cup \{q_u\}$, and under-approximate the goal by $Q_G = \{q \in Q \mid q \subset \check{z}_G \}.$ 

\begin{remark} [Non-uniform Grid] \label{rem:grid}
    We start with a uniform grid that covers the latent domain $Z$ and then add finer grid cells near the boundary of the convex hulls to more closely under-approximate the non-axis aligned domain in the latent space. We then identify grid cells that overlap with mapped regions of interest and make a finer grid around their boundaries such that the sets $Q_U$ and $Q_G$, while conservative, will be less conservative than if we used a uniform grid.
\end{remark}

When refining, we choose states in $Q_?$ that have transitions to either $Q_{yes}$ or $Q_{no}$ and evenly split all dimensions of $Z$ in that region.

\subsubsection{3D example}
The dynamics for the 3D system are as follows:
\begin{subequations}
    \begin{align}
        \theta(a) &= \frac{\pi}{7}\cos{a}\\
        \begin{bmatrix} x^{(1)}_{k+1} \\
            x^{(2)}_{k+1} \\
            x^{(3)}_{k+1} \end{bmatrix} &= \begin{bmatrix}0.7\left(x^{(1)}_{k} \cos{\theta(x^{(3)}_{k})} - x^{(2)}_{k} \sin{\theta(x^{(3)}_{k})}\right) \\
            0.7\left(x^{(1)}_{k} \sin{\theta(x^{(3)}_{k})} + x^{(2)}_{k} \cos{\theta(x^{(3)}_{k})}\right)\\
            0.7 x^{(3)}_{k}       
        \end{bmatrix}
    \end{align}
\end{subequations}
We use 10000 data points to learn the encoder and the Deep Kernel network, and 1000 points for the inclusion GP predictions. The encoder is a five layer network with size $\reals^{n_x} \rightarrow 512 \rightarrow 256 \rightarrow 128 \rightarrow 64 \rightarrow \reals^{n_p}$ with a skip layer adding at the $256$ layer.
It takes 127.42 seconds to learn the encoder, 16.74 seconds for mapping regions of interest and the domain, and 43.48 seconds to learn the latent inclusion deep kernel model. We spend 96.01 seconds bounding the deep kernel neural network and then 436.41 seconds bounding the base kernel after to identify the sets $Im(q,u)$. Then 22.11 seconds are spent to identify the transition function in the latent space and then less then 1 second is spent on verification. 


\subsubsection{6D example}
The over-parameterized dynamics are as follows:
\begin{align}
    \begin{bmatrix} x^{(1)}_{k+1} \\
    x^{(2)}_{k+1} \\
    x^{(3)}_{k+1} \\
    x^{(4)}_{k+1} \\
    x^{(5)}_{k+1}\\
    x^{(6)}_{k+1}
    \end{bmatrix} = \begin{bmatrix}0.7\left(x^{(1)}_{k} x^{(5)}_k - x^{(2)}_{k} x^{(6)}_k\right) \\
    0.7\left(x^{(1)}_{k} x^{(6)}_k + x^{(2)}_{k} x^{(5)}_k\right)\\
    0.7 x^{(3)}_{k}       \\
    \cos{0.7 x^{(3)}_k} \\
    \cos (\frac{\pi}{7} \cos(0.7 \cos^{-1}(x^{(4)}_k))) \\
    \sin (\frac{\pi}{7} \cos(0.7 \cos^{-1}(x^{(4)}_k))) 
    \end{bmatrix}
\end{align}

Note that this follows the 3D dynamics with observations of sin and cos terms. 
We use 15000 data points to learn the encoder, 10000 for the Deep Kernel network, and 1000 points for the inclusion GP predictions. 

When initializing in $Q_{no}$ we find roughly $40\%$ of trajectories collide with the obstacle and the remaining trajectories are safe and reach the goal. This conservatism comes from over-approximating the unsafe set in the latent space.

\subsubsection{26D example}
For the LiDAR system, the closed loop dynamics rely on a 26D observation, without LiDAR inputs it is known that this is an unobservable system for predicting the closed loop dynamics. We discretized the single integrator with a time step of 0.4 seconds.
We use 10000 data points to learn the encoder, 30000 for the Deep Kernel network, and again use 1000 points for the IGP prediction.
It takes 515.298 seconds to learn the encoder 
and 67.47 seconds to learn the latent deep kernel. 
We spend 681.23 seconds to construct the initial latent abstraction and two hours on 10 iterative refinements, with the majority of the time being spent on bounding the outputs of the Inclusion GP model. For each refinement less then 1 second is spent on verification. 

As the original function is not smooth (discontinuities in the LiDAR observations occur in domains that are non-convex due to obstacle placement), the function $\enc(f(x))$ technically can't be modeled with the SE kernel, however $\tilde{g}(z,c)$ must be a continuous function on the input domain $Z \times C$ and hence it can be modeled with the SE kernel. Effectively, there are regions of the parameter space $c \in C$ that do not correspond with the dynamics of the original system, hence we verify a dynamics model for $g$ as a conservative connected set that contains the discontinuous evolution of $\enc(f(\cdot))$. This results in added conservatism in the abstraction, which can be seen with the unverifiable regions, the boundary of which corresponds to regions in the original space that have discontinuities in the dynamics. 
Alternate kernels that allow for discontinuities may improve the results. It is also possible that using local GP models for each region in which the dynamics are continuous would allow for a better abstraction as the system and approaches such as Work \cite{park2022jump} could be used, but local GP regression is not explored in this work. We also stress that since the original system is discontinuous, a traditional neural network architecture cannot capture the evolution of the dynamics correctly (as common architectures can only model continuous systems), which further demonstrates the strength of our approach.


\subsection{Loss Ablation study} \label{app:ablation}
Here we show results on an ablation study on the importance of each loss term when using the $6D \rightarrow 2D$ dataset. We train under two different seeds to assess consistency.
If we keep only $L_1$ and not $L_2$ we train with MSE on predictions rather than as a set valued neural network as $L_1$ can be trivially satisfied if $L_2$ isn't considered. If the Latent Domain Size is NA, that means the mapped domain collapsed to either a line or a point, in either case abstraction is impossible. We report the final loss of the network, in each case using the same hyper-parameters and network architecture, as well as an approximate range identified for the domain in the latent space in Table~\ref{table:ablation}.

Results identify that our loss terms represent competing objectives, as a smaller loss does not necessarily imply a domain that is useful for abstraction. This also results in multiple local minima, which can be seen from tests where the encoder identifies a potentially useful domain in one seed and a domain that is infeasible for abstraction under another. A detailed study on hyper-parameters for training has not been performed. 

Tests 2 and 3 show the importance of utilizing a set-valued neural network during training for dynamics (i.e. $L_1$ and $L_2$). While Test 2 uses standard MSE error on dynamics prediction and can find a latent space where discretization can be performed, Test 3 demonstrates that the dynamics likely have a very large set posterior as attempting to minimize the size of the set results in a latent space with nearly no volume. This suggests that the latent space in Test 2 would be very ill suited for accurate dynamics modeling and an abstraction would have a dense transition structure.

It also becomes clear that without both the VAE and "similar norm" ($L_4$ and $L_5$) terms, the latent space is either not suited for abstraction (either excessively large - Test 1 - or collapsed to a point or line - Tests 3, 5, 8, 11, 12) or the training is unstable and feasibility of finding an interpretable latent space becomes unclear (Tests 4, 10). Only tests $6, 7, 9$ consistently produce latent spaces where abstraction may be feasible. From test $4$ it can be seen that while $7$ (which does not include the dynamics loss term) produces a latent space that is interpretable there is no information about the dynamics and attempting to include this information immediately results in instability without the VAE term.

\begin{table} [h]
    \caption{Ablation study results when training on the 6D dataset with different loss terms weighted at 0. The first 5 columns identify the multiplier applied to $\alpha_i$ (i.e. 0 means the loss term is ignored). We run two tests for each setup under different seeds.}
    \label{table:ablation}
    \centering
    \begin{tabular} {l | c c c c c | c c}
            \toprule
            Test & $L_1$ & $L_2$ & $L_3$ & $L_4$ & $L_5$ & Loss & Latent Domain Size ($Z \subset$ ranges) \\
            \midrule
            \multirow{2}{*}{$1$} & 0 & 0 & 1 & 0 & 0 & 0.02935 & $[0, 1600] \times [0, 800]$\\
            & 0 & 0 & 1 & 0 & 0 & 0.02367 & $[0, 1000] \times [0, 1400]$\\
            \midrule
            \multirow{2}{*}{$2$} & 1 & 0 & 1 & 0 & 0 & 0.05373 & $[2, 16] \times [0, 27]$\\
            & 1 & 0 & 1 & 0 & 0 & 0.05228 & $[2, 18] \times [0, 27]$\\
            \midrule
            \multirow{2}{*}{$3$} & 1 & 1 & 1 & 0 & 0 & 0.19286 & NA \\  
            & 1 & 1 & 1 & 0 & 0 & 0.24785 & NA\\
            \midrule
            \multirow{2}{*}{$4$} & 1 & 1 & 1 & 0 & 1 & 0.26764 & $[-0.7, 1.6] \times [-0.3, 1.6]$\\
            & 1 & 1 & 1 & 0 & 1 & 0.44686 & NA\\
            \midrule
            \multirow{2}{*}{$5$} & 1 & 1 & 1 & 1 & 0 & 0.20098 & NA \\
            & 1 & 1 & 1 & 1 & 0 & 0.18932 & NA  \\  
            \midrule
            \multirow{2}{*}{$6$} & 0 & 0 & 1 & 1 & 1 & 0.36221 & $[-1, 1] \times [-1, 1]$ \\
            & 0 & 0 & 1 & 1 & 1 & 0.36680 & $[-1, 1.2] \times [-1, 1.2]$ \\
            \midrule
            \multirow{2}{*}{$7$} & 0 & 0 & 1 & 0 & 1 & 0.18156 & $[6.2, 8.5] \times [3, 5.5]$ \\
            & 0 & 0 & 1 & 0 & 1 & 0.19252 & $[7.5, 10] \times [3.3, 6]$ \\
            \midrule
            \multirow{2}{*}{$8$} & 0 & 0 & 1 & 1 & 0 & 0.19006 & NA \\
            & 0 & 0 & 1 & 1 & 0 & 0.18701 & NA \\  
            \midrule
            \multirow{2}{*}{$9$} & 1 & 1 & 0 & 1 & 1 & 0.21305 & $[-1, 1] \times [-0.7, 1]$ \\
            & 1 & 1 & 0 & 1 & 1 & 0.24344 & $[-0.7, 1.2] \times [-0.7, 1.2]$ \\
            \midrule
            \multirow{2}{*}{$10$} & 1 & 1 & 0 & 0 & 1 & 0.31367 & NA \\
            & 1 & 1 & 0 & 0 & 1 & 0.11487 & $[-0.5, 1.6] \times [-1, 1.6]$ \\
            \midrule
            \multirow{2}{*}{$11$} & 1 & 1 & 0 & 1 & 0 & 0.00703 & NA  \\ 
            & 1 & 1 & 0 & 1 & 0 & 0.04908 & NA \\
            \midrule
            \multirow{2}{*}{$12$} & 1 & 1 & 0 & 0 & 0 & 0.02646 & NA \\
            & 1 & 1 & 0 & 0 & 0 & 0.00752 & NA \\
            \bottomrule
    \end{tabular} 

\end{table}

\clearpage
\section{Proofs}
    \label{sec:app_proofs}
    
\subsection{Proof of Lemma~\ref{lemma: connected pre and post image} (connected inclusion dynamics)} \label{app:CICO}

First we prove that an ICNN with the additional constraints listed in Section \ref{sec:encoder architecture} will have connected pre-images.

The pre-image of a function $f$ from a point $y$ is the set $\{x \mid f(x) = y\} = f^{-1}(y)$. This set is referred to as a fiber. With an abuse of notation we use $()^{-1}$ to denote the pre-image of a function (i.e. a non invertible matrix still has a well defined pre-image).

Let $A_1: \reals^n \rightarrow \reals^t$ be an injective affine map (i.e. $t \geq n$), let $\phi: \reals^t \rightarrow \reals^t$ be a coordinate-wise strictly monotone, convex continuous map (hence a homeomorphism onto its image), let $A_2: \reals^t \rightarrow \reals^m$ be an affine map whose linear part has full row rank $m$ (so $A_2$ is surjective), and let $A_3: \reals^n \rightarrow \reals^m$ be an injective affine map (i.e. $m \geq n$, hence $t \geq m \geq n$).
Define
\begin{align}
    s(x) = A_2(\phi(A_1(x))) + A_3(x)
\end{align}
then, for every $y \in \reals^m$ the fiber $s^{-1}(\{y\})$ is connected. 

Let $s: \reals^{n} \rightarrow \reals^m$ be the encoder up to the first skip layer, we note that this defines a skip to layer two, but the argument holds for deeper skip layers under the same logic.
Write $f(x) = A_2(\phi(A_1(x)))$ and $g(x) = A_3(x)$, then $s = f + g$. We prove two facts and then combine them.

First, every fiber of $f$ is convex. Let $S := A_1(\reals^n) \subset \reals^t$. Since $A_1$ is injective and affine, $S$ is an $n$-dimensional affine subspace of $\reals^t$. As pre- and post- composing affine homeomorphism does not break the connectedness of fibers, we can assume without loss of generality that the affine subspace $S$ is shifted through an affine transformation $T$ to $\{(u_1, \ldots, u_n, 0, \ldots, 0)\}$ and then reason over this shifted coordinate space. Since $\phi$ is coordinate-wise strictly monotone, the mapping of the shifted subspace
\begin{align}
    \Phi := \phi|_{TS} : TS \rightarrow M := \phi(TS)
\end{align}
is also a homeomorphism from the affine subspace $S$ onto the set $M \subset \reals^t$. Thus $M$ is homeomorphic to $\reals^n$.
Fix $y\in \reals^m$, the fiber of $f$ over $y$ is
\begin{align}
    f^{-1}(\{y\}) = \{x \in \reals^n: A_2(u) = y \>, \> u=\phi(A_1(x)) \in M\}
\end{align}
Equivalently
\begin{align}
    f^{-1}(\{y\}) = \{u \in M: A_2(u) = y \} = M \cap L_y
\end{align}
where $L_y := \{u \in \reals^t: A_2(u) = y\}$ is an affine subspace of $\reals^t$. Observe that $M$ is the continuous image of the affine space $TS$ under a coordinate-wise strictly monotone map, hence $M$ is a convex set defined with intervals over the first $n$ dimensions and constants for the next $t-n$ dimensions. The intersection of convex sets is convex and hence connected. Therefore $f$ has connected pre-images.

Now we consider adding an injective affine map and show that this preserves connected fibers.
Fix $y\in \reals^m$, the fiber of $s$ over $y$ is
\begin{align}
    s^{-1}(\{y\}) = \{x: f(x) + g(x) = y\} = \{x : f(x) = y - g(x)\}
\end{align}
The injectivity of $g$ implies that $g$ is a homeomorphism onto its image ($S_3$ an $n$-dimensional affine subspace of $\reals^m$). Sine $g$ is an affine homeomorphism onto $S_3$, $x = A_3^{-1}(u)$ for $u \in S_3$. We define a new function
\begin{align}
    \tilde{f}(u) := f(A_3^{-1}(u)) \> \text{for} \> u \in S_3
\end{align}
Since $A_3^{-1}$ is a homeomorphism and every fiber of $f$ is connected, every fiber of $\tilde{f}$ is connected.
Then the fiber condition for $s$ becomes $\tilde{f}(u) + u = y$.
Hence,
\begin{align}
    s^{-1}(\{y\}) = A_3^{-1}(\{u \in S_3: \tilde{f}(u) + u = y\})
\end{align}
Again since $g$ is an affine homeomorphism, connected sets in $S_3$ correspond bijectively to connected sets in $\reals^n$. Hence, the solution set to $\tilde{f}(u) + u = y$ is the continuous pre-image of a convex intersection as before and hence connected.

We note that $s(x)$ defines an ICNN with a skip layer after the first activation. Hence, by induction, all skip layers have connected pre-images. Continuity of the system enforces that the pre-image of a connected set in the affine subspace will be defined by the union of convex sets, which may be non-convex in general but will still be connected.

We can now consider the network after skip layers. Let $T: \reals^k \rightarrow \reals^p$ be the final layers of the network where $k \geq n \geq p$ as the last skip layer must remain injective. $T$ is composed of monotone, full rank, affine transformations that reduce dimensionality and strictly monotone non-linear activations. Each affine transformation, by definition, has a connected pre-image. Strictly monotone non-linear activations are homeomorphism, and hence also have connected pre-images. Therefore $T$ has a connected pre-image for any point. Following arguments from $s$, the pre-image of $T$ under a point will be a convex set. The intersection of a convex set with an affine subspace will be a connected set in the affine subspace, hence the entire network is guaranteed to have a connected pre-image from a point, resulting in a CICO network.

Then, by Lemma~\ref{lemma: connected pre and post image}, $\X_z$ is a connected set. Further, by continuity of $f$, $\X'_z$ is also connected. The continuity of $\enc$ then ensures connectivity of $\Z'$.

A simple condition to check is if the Jacobian of the network has rank $n_p$. If this is satisfied, along with the architectural constraints, then the fibers are connected as the network is a proper submersion.

\subsection{Proof of Lemma \ref{lemma: correct state id} (Sound Reach-Avoid Abstraction States)} \label{app:correct_label}

Recall $Z$ defines an under-approximation of domain $X$ in the latent space. Also recall that $\hat{r}_{Z} := \text{Conv}(D_{Z}) \oplus \mathcal{B}_{\epsilon(N, \delta)}(0)$ where $D_{Z} = \{\enc(x_j) \mid x_j \sim r\}_{j=1}^N$
defines a high confidence over-approximation in the latent space for region $r$ by Proposition \ref{prop:randup}, that is $\hat{r_{Z}}$ defines a convex set in the latent space that will contain all points $\post_{\enc}(r)$ with high confidence. 
Hence by defining $Q_U$ according to $\hat{z}_{u,i}$, we have a high-confidence over-approximation of the set $X_U$ in the latent space abstraction.
Similarly, since $\check{z}_{G} = \text{Conv}(\{\enc(x_j) \mid x_j \sim X_G\}_{j=1}^N)$ is an under-approximation of the where the goal can map to in the latent space, and $\cup_k \hat{z}_{\lnot G, k}$ is a high-confidence over-approximation of $X \setminus X_G$, defining $Q_G$ as $\check{z}_{G} \setminus \cup_k \hat{z}_{\lnot G, k}$ we guarantee that $Q_G$ represents points that only the goal can map to with high-confidence.



\subsection{Proof of Theorem~\ref{thm: transition function construction} (Inclusion GP Error Bounds)}
For notational simplicity, we will drop $(\cdot)^{(j)}$ in the proof and work under the assumption that $g$ outputs a set in $\reals$ without loss of generality and take $C = [0, 1]$.

Let $z_i \in Z \subset \reals^{n_p}$ and there exists a function
$$\tilde{g} : \reals^{n_p} \times C \rightarrow \reals^{n_p}$$
such that
$$z_{i+1} = \tilde{g}(z_i, c_i)$$
for some unknown $c_i \in C$ and for some $\hat{x} \in \text{Pre}_{\enc}(z_i)$ then $z_{i+1} = \enc(f(\hat{x}))$. Assume $\tilde{g} \in \mathcal{H}_{\kappa}$ an RKHS with kernel over $(z,c)$, and $c$ is an accurately learned parameterization.

Then there's the true set map
$$g(z) := \{\tilde{g}(z, c) \mid c \in C\}$$
and the learned set map
$$\hat{g}(z) := \{\mu(z,c) \mid c \in C\}.$$
Then
\begin{align}
   d_H(g(z), \hat{g}(z)) \leq \sup_{c\in C}|\tilde{g}(z, c) - \mu(z,c)| \leq \overline{\epsilon}(z). 
\end{align}

First, we prove functional equivalence in the latent space under reparameterization. We assert that there exists some function $\mathcal{T}: \reals^{n_x} \rightarrow \reals^{n_p}$ that is equivalent to the mapping $\enc(f(x))$. $\mathcal{T}$ can be defined as $\tilde{g}(\enc(x), \psi_c(x))$ where $\psi_c: \reals^{n_x} \rightarrow C$ acts as an additional unobserved coordinate in the latent space that resolves where in $g(z)$ the system evolves to.
The existence of such a $\mathcal{T}$ is trivially true, e.g. with latent dynamics model $\tilde{g}$, e.g. $\tilde{g}(z, c) = c$ and $\psi_c(x) = \enc(f(x))$. However, there are infinitely many functionally equivalent representations of $\mathcal{T}$ as $\tilde{g}$ and $\psi_c$ do not have a fixed form.

Under the assumption $\mathcal{T}$ is in the RKHS of deep kernel $\kappa(\psi_l(\enc(x), \psi_c(x)))$, then a functionally equivalent $\tilde{g}$ is in the RKHS of $\kappa(z, c)$ under the coordinates defined by $\enc$ and $\psi_c$. $\mathcal{T}$ can be shown to be arbitrarily close to a function in the RKHS of the deep kernel under the assumption that the learned networks approximately retain the level sets of $\mathcal{T}$, which can be achieved by training the networks to predict the outputs of $\mathcal{T}$. Learning the kernel in this method to model $\mathcal{T}$ implicitly defines latent data $\{(z_i, c_i, z_i^+)\}_{i=1}^m$, which can be used as the predictive dataset in the latent space.

Then the proof that $d_H(g(z), \hat{g}(z)) \leq \overline{\epsilon}(z)$ is straight-forward.
By definition $\epsilon(z, c) \leq \overline{\epsilon}(z)$, then according to Proposition \ref{our_det_theorem} for all $c$ 
\begin{align}
    |\mu_{D}(z, c) - \tilde{g}(z, c)| \leq \overline{\epsilon}(z),
\end{align}
hence $\sup_c |\mu_{D}(z, c) - \tilde{g}(z, c)| \leq \overline{\epsilon}(z)$ which satisfies the requirement that $d_H(g(z), \hat{g}(z)) \leq \overline{\epsilon}(z)$.

Let $\text{Im}(z) = \{v \mid v = \tilde{g}(z, c) \> \forall c \in C\}$. Then it follows that $|\check{\mu}_{D}(z) - \min \text{Im}(z)| \leq \overline{\epsilon}(z)$ and similarly with the maximum. Then under the assumption that $g := \{\tilde{g} \mid c \in C\}$ , $g$ is contained in the union of all error sets and the theorem is complete.

\subsection{Proof of Theorem \ref{thm: soundness} } \label{app:soundness}

Define a simulation relation as follows:
\begin{definition}[High-Confidence Simulation Relation] 
    Let $\mathcal{R} \subseteq \reals^{n_x} \times Q$ be the relation defined by $(x,q) \in \mathcal{R}$ iff $\enc(x) \in q$. 
    The NTS $\mathcal{N}$ 
    is said to simulate 
    System~\eqref{true_dynamics} on domain $X$ w.r.t. $\varphi = (X,X_G,X_U)$ with confidence $1-\delta$, if for every trajectory $\omega_{x_0}$ with $(x_0, q_0) \in \mathcal{R}$, there exists a trajectory $\omega_{q_0} \in \Omega_{q_0}$ such that, for all $k \geq 0$
    \begin{enumerate}[label=\arabic*.]
        \item $(\omega_{x_0}(k), \omega_{q_0}(k)) \in \mathcal{R}$,
        \item $\omega_{q_0}(k) \in Q_G \implies \mathbb{P}(\omega_{x_0}(k) \in X_G) \geq 1-\delta$,
        \item $\omega_{x_0}(k) \in X_U \implies \mathbb{P}(\omega_{q_0}(k) \in Q_U) \geq 1-\delta$.
    \end{enumerate}
\end{definition}

By Lemma~\ref{lemma: correct state id}, $(x, Q_G) \in \mathcal{R}$ holds only if $x \in X_G$ with probability at least $(1-\delta)$ and if $x \in X_U$ then $(x, Q_U) \in \mathcal{R}$ with probability at least $(1-\delta)$. By Proposition \ref{prop:trans}, for every transition $x' = f(x)$ if $(x,q) \in \mathcal{R}$ then $\exists q'$ such that $T(q, q') = 1$ and $(x',q') \in \mathcal{R}$. This satisfies the requirements for a simulation relation.
    
We then prove the contrapositive, 
$$\omega_{x_0} \not\models \varphi \implies \exists \omega_{q_0} \not\models \phi.$$
That is, any counter example on System \eqref{true_dynamics} must also have a counter example in the NTS. 
A sound abstraction ensures $\exists \omega_{q_0}$ that remains in relation with $\omega_{x_0}$ for all $i \geq 0$.

Assume $\omega_{x_0} \not\models \varphi$, i.e. 
(i) $\exists i \in \nats$ s.t. $\omega_{x_0}(i) \in X_U$ and $\not\exists j < i$ s.t. $\omega_{x_0}(j) \in X_G$ or (ii) $\not\exists i \in \nats$ s.t. $\omega_{x_0}(i) \in X_G$. For (i) if $\omega_{x_0}(i) \in X_U$ the probabilistic simulation relation implies that $\omega_{q_0}(i) \in Q_U$ with probability at least $(1-\delta)$, satisfying the contrapositive with probability at least $(1-\delta)$. For (ii) this implies $\omega_{q_0}(i) \notin Q_G$ for any $i$, which holds deterministically if $\omega_{x_0}(i) \notin X_G$ and again satisfies the contrapositive. 

Then $\omega_{x_0} \not\models \varphi$ implies $\exists \omega_{q_0} \not\models \phi$ with probability at least $(1-\delta)$. This establishes the contrapositive and hence $\Omega_q \models \phi \implies \omega_x \models \varphi$ with probability at least $(1-\delta)$.

\section{Extension to LTL}
    \label{sec:app_LTL}
    \subsection{Overview of LTL and NNF}

An infinite \textit{trajectory} of System \eqref{true_dynamics} is written as $\omega_\px = \px(0)  \px(1) \ldots$ and the set of all infinite trajectories is denoted as $\Omega_\px$.
We are interested in the temporal properties of System~\eqref{true_dynamics} trajectories in a subset $X \subset \reals^{n_x}$ 
w.r.t. a set of regions $R = \{r_1, \ldots, r_l\}$, where $r_i \subseteq X$.
In this work, we assume sets $r_1, \ldots, r_l$, and $X$ are convex.

To specify and reason about temporal properties, we first define a set of atomic proposition $\Pi = \{\prop_1, \ldots, \prop_l\}$, where $\prop_i$ is true iff $\px \in r_i$. Let $L: X \to 2^\Pi$ be a labeling function that assigns to each state the set of atomic propositions that are true at that state. Then, the \textit{observation trace} of trajectory $\omega_\px$ is $\rho = \rho_0 \rho_1 \ldots$ where $\rho_k = L(\omega_\px \! (k))$ for all $k \geq 0$. To express the desired temporal properties, we use Linear Temporal Logic (LTL)~\cite{baier2008principles}.

\begin{definition}[LTL Syntax]
    \label{def:LTL}
    Given a set of atomic propositions $\Pi$, an LTL formula is defined recursively as
    \begin{equation*}
        \varphi = \prop \mid \neg \varphi \mid \varphi \vee \varphi \mid \bigcirc \varphi \mid  \varphi \U \varphi 
    \end{equation*}
    where $\prop \in \Pi$, $\lnot$ ("not") and $\vee$ ("or") are Boolean operators, and $\bigcirc$ ("next") and $\U$ ("until") are temporal operators.
\end{definition}
\noindent
The operators $\wedge$ ("and"), $\square$ ("always"), and $\Diamond$ ("eventually") can be defined using $\lnot$, $\vee$, and $\U$. 

The semantics of LTL are defined over infinite traces~\cite{baier2008principles}
and provided in Appendix~\ref{app:NNF}.
We say infinite trajectory $\omega_\px$ satisfies formula $\varphi$, denoted by $\omega_\px \models \varphi$, if its observation trace satisfies $\varphi$.

\subsubsection{Semantics of LTL and NNF}
\label{app:NNF}

The semantics for an LTL specification are defined inductively as follows.
\begin{align*}
    \omega_{x0} \models p \> &\iff  p \in L(x_0) \\
    \omega_{x0} \models \neg \varphi &\iff \omega_{x0} \not\models \varphi \\
    \omega_{x0} \models \varphi \vee \psi &\iff  \omega_{x0} \models \varphi \text{ or } \omega_{x0} \models \psi \\
    \omega_{x0} \models \varphi \wedge \psi &\iff  \omega_{x0} \models \varphi \text{ and } \omega_{x0} \models \psi \\
    \omega_{x0} \models \bigcirc \varphi &\iff \omega_{x1} \models \varphi \\
    \omega_{x0} \models \varphi \U \psi &\iff \exists j \geq 0 \text{ s.t. } \omega_{xj} \models \psi \text{ and } \forall k < j \> \omega_{xk} \models \varphi \\
    \omega_{x0} \models \Diamond \varphi &\iff \exists j \geq 0 \text{ s.t. } \> \omega_{xj} \models \varphi \\
    \omega_{x0} \models \square \varphi &\iff \forall j \geq 0 \> \omega_{xj} \models \varphi 
\end{align*}
As noted, $\square$, $\Diamond$, and $\wedge$ can be defined using $\lnot$, $\U$, and $\vee$, but to define Negation Normal Form (NNF), the semantics of $\square$, $\Diamond$, and $\wedge$ must be made clear.

Any LTL formula in NNF is defined inductively as 
\begin{align*}
    \mathit{NNF}(p) &= p\\
    \mathit{NNF}(\lnot p) &= \lnot p\\
    \mathit{NNF}( \varphi \vee \psi) &= \mathit{NNF} \varphi \vee \> \mathit{NNF} \psi \\
    \mathit{NNF}( \varphi \wedge \psi) &= \mathit{NNF} \varphi \wedge \> \mathit{NNF} \psi \\
    \mathit{NNF}(\lnot ( \varphi \vee \psi)) &= \mathit{NNF}(\lnot \varphi \wedge \lnot \psi) \\
    \mathit{NNF}(\lnot ( \varphi \wedge \psi)) &= \mathit{NNF}(\lnot \varphi \vee \lnot \psi) \\
    \mathit{NNF}(\bigcirc \varphi ) &= \bigcirc \mathit{NNF}( \varphi ) \\
    \mathit{NNF}(\lnot \bigcirc \varphi ) &= \bigcirc \mathit{NNF}( \lnot \varphi ) \\
    \mathit{NNF}(\varphi \U \psi) &= \mathit{NNF}( \varphi ) \U \> \mathit{NNF}( \psi ) \\
    \mathit{NNF}(\lnot(\varphi \U \psi)) &= \mathit{NNF}( \lnot \varphi ) R \> \mathit{NNF}( \lnot \psi ) \\
    \mathit{NNF}(\Diamond \varphi ) &= \Diamond \mathit{NNF}( \varphi ) \\
    \mathit{NNF}(\lnot \Diamond \varphi ) &= \square \mathit{NNF}( \lnot \varphi ) \\
    \mathit{NNF}(\square \varphi ) &= \square \mathit{NNF}( \varphi ) \\
    \mathit{NNF}(\lnot \square \varphi ) &= \Diamond \mathit{NNF}(\lnot \varphi )
\end{align*}
where $R$ is the "release" operator defined as $\varphi R \psi \equiv \lnot( \lnot \varphi U \lnot \psi)$.

Then if $\overline{\varphi}$ has no negations on any proposition, as it is defined over $p, p_n$, $\lnot \overline{\varphi}$ will have negations on all propositions.

\subsection{NTS and Labeling for LTL Specifications} \label{app:LTL label}
For LTL specifications $\varphi$, and the associated atomic propositions, the NTS must be defined with a labeling function.
\begin{definition}
    [NTS] A labeled Transition System (TS) is a tuple $\mathcal{N} = (Q, T, L, AP)$, where 
    $Q$ is a finite set of states, 
    $T: Q \times Q \rightarrow \{0, 1\}$ is a transition function such that $T(q ,q') = 1$ if $q'$ is a successor state of state $q \in Q$, 
    $AP$ is a set of atomic propositions, and 
    $L: Q \rightarrow 2^{AP}$ is a labeling function that assigns to each $q \in Q$ a subset of $AP$.
    $\mathcal{N}$ is called \emph{Non-deterministic} (NTS) if $\sum_{q'\in Q} T(q,q') > 1$.
\end{definition}
We say $\omega_q \models \varphi$ if its observation trace $\rho = \rho_0\rho_1\ldots$ where $\rho_k = L(\omega_q \! (k))$ satisfies $\varphi$. Then $\Omega_q \models \varphi$ iff $\omega_q \models \varphi$ for all $\omega_q \in \Omega_q$.

Now labeling must be done for each region of interest in $R$ to define $L$, rather than just for $G$ and $O$. This is where the most complexity comes in, and results in an alternate lemma for the correctness of labels, while the rest of the approach remains identical.

To ensure that the labeling for each region $r_i \in R$ is done correctly in the latent space, we follow the method outlined in \citet{cauchi2019efficiency} and begin by representing $r_i$ and its complement $r_{n,i} = X \setminus r_i$ with new atomic propositions $p_i, p_{n,i}$, where $p_{n,i}$ then represents $\neg p_i$. The atomic propositions for the NTS are then defined as 
\begin{align}
    \label{eq: atomic prop AP}
    AP := \Pi \cup \{p_{n,1}, \ldots, p_{n,l}\}.
\end{align}

To enable mapping of the regions with correct under- and over-approximations for the purpose of labeling,
we first partition $r_{n,i}$ into a set of convex regions $r_{n,i}^k$ such that $r_{n,i} = \cup_k r_{n,i}^k$. Then, we construct labeling regions that correspond to $r_i$ in the latent space according to
\begin{subequations}
    \begin{align}
        \check{r}_{Z,i} = c_{Z,i} \setminus \cup_k r_{Z,n,i}^k, 
        \label{eq:under_approx} \\
        \hat{r}_{Z,i} = c_{Z,i} \oplus \mathcal{B}_{\epsilon(N, \delta)}(0), \label{eq:over_approx}
    \end{align}
    where
    \begin{align}
        &c_{Z,i} = \text{Conv}(\{\enc(x) | x \sim r_i\}_{i=1}^N), \\
        &r_{Z,n,i}^k = \text{Conv}(\{\enc(x) | x \sim r_{n,i}^k\}_{i=1}^N) \oplus \mathcal{B}_{\epsilon(N, \delta)}(0).
    \end{align}
\end{subequations}
Note that the under-approximation $\check{r}_{Z,i}$ is defined using an over-approximation of $X \setminus r_i$. We use $\check{r}_{Z,i}$ to define the set of states with label $p_i$. This is necessary to ensure that no state in $X \setminus r_i$ receives the label $p_i$ when mapped to the latent space. Intuitively, $c_{Z,i}$ is an under-approximation of where $r_i$ can map to, but not an under-approximation of where $r_i$ \emph{exclusively} maps to.

This construction guarantees we under-approximate both the region $r_i$ and its negation in the latent space, as well as the set of states in the NTS that are associated with that region.
We define the NTS labeling function $L: Q \to 2^{AP}$ 
such that $p \in L(q)$ if $q \subset \check{r}_{Z,i}$, and $p_{n,i} \in L(q)$ if $q \subset Z \setminus \hat{r}_{Z,i}$. 

\begin{lemma}
    [Conservative Labeling] 
    \label{lemma: conservative labeling}
    For a point $x \in X$, let $q_x \ni \enc(x)$ denote the latent region in the partition $Q$ that contains $\enc(x)$, and with an abuse of notation, let $L: X \to 2^{AP}$ or $L:Q \to 2^{AP}$, where 
    $AP := \Pi \cup \{p_{n,1}, \ldots, p_{n,l}\}$.
    Under the construction of the latent space and labeling procedure described above, 
    then $L(q_x)\subseteq L(x)$ holds with confidence at least $(1-\delta)^{|AP|-|L(x)|}$. 
\end{lemma}

Hence we capture (conservatively) all of the regions $R$ (and their negations) by the labels in $\mathcal{N}$.
We first translate $\varphi$ to NNF as $\mathit{NNF}(\varphi)$ (see Appendix \ref{app:NNF} for details), and then replace every $\neg p_i$ in $\mathit{NNF}(\varphi)$ with $p_{n,i}$. Then any trace that satisfies $\overline{\varphi}$ over $AP$ will also satisfy $\varphi$ over $\Pi$ \citep{cauchi2019efficiency}.

Recall $Z$ defines an under-approximation of domain $X$ in the latent space. Also recall that Eq. \eqref{eq:over_approx} defines a high confidence over-approximation in the latent space for region $r_i$ by Proposition \ref{prop:randup}, that is $\hat{r_{Z,i}}$ defines a convex set in the latent space that will contain all points $\{\enc{x} \mid x \in r_i\}$ with high confidence, therefore $Z \setminus \hat{r_{Z,i}}$ is an under-approximation of $X \setminus r_i$ in the latent domain $Z$. Also recall that in the original space $p_{n,i} \in L(x \in X \setminus r_i)$, then by labeling $Z \setminus \hat{r_{Z,i}}$ with $p_{n,i}$ we under-approximate the label in the latent space. The logic follows similarly for labels $p_i$, as we first over-approximate $X \setminus r_i$ in the latent space. This implies that if $p_i \in L(q_x)$ then $p_i \in L(x)$ for $\enc(x) \in q_x$, hence $L(q_x) \subseteq L(x)$. Since each atomic proposition is under-approximated with confidence $\geq 1-\delta$ by over-approximating its negation then, by union bound, the label $L(q_x)\subseteq L(x)$ holds with confidence $\geq (1-\delta)^{|AP|-|L(x)|}$. 

\subsection{Abstraction Soundness for LTL}

We can then re-define Sound Abstraction as follows:
\begin{definition}
    [Sound Abstraction] Let $\mathcal{R} \subseteq \reals^{n_x} \times Q$ define a relation such that $(x,q) \in \mathcal{R}$ if $\enc(x) \in q$. NTS $\mathcal{N}$ is a sound abstraction of System \eqref{true_dynamics} in domain $X$, if for any trajectory $\omega_{x_0}$, $\exists \omega_{q_0} \in \Omega_{q_0}$ s.t. $(\omega_{x_0}(k), \omega_{q_0}(k)) \in \mathcal{R}$ and $L(\omega_{q_0}(k)) \subseteq L(\omega_{x_0}(k))$ holds with confidence at least $(1-\delta)^{|AP|-|L(x)|}$ for all $k \geq 0$.
\end{definition}

\begin{theorem}
    [Sound Abstraction] 
    \label{thm: soundness LTL}
    Let NTS $\mathcal{N}$ be constructed with the states and labeling function defined according to Section~\ref{app:LTL label}, and with a transition function defined according to Proposition~\ref{prop:trans}. Then $\mathcal{N}$ is a sound abstraction for System~\eqref{true_dynamics}, and for every $(x,q) \in \mathcal{R}$, if all the paths initialized at $q$ in $\mathcal{N}$ satisfy $\overline \varphi$, then the trajectory $\omega_x$ is guaranteed to satisfy $\overline \varphi$, i.e.,
    \begin{align}
        \Omega_q \models \overline{\varphi} \; \implies \; \omega_x \models \overline{\varphi}  \qquad \forall (x,q) \in \mathcal{R}.\label{thm:sound_eq}
    \end{align}
\end{theorem}

Lemma~\ref{lemma: conservative labeling} and Proposition~\ref{prop:trans} establish $\mathcal{N}$ as a sound abstraction. Then it can be seen that if $\omega_{x_0} \not\models \overline{\varphi}$ there must exist a path $\omega_{q_0}$ where $(x_i, q_i) \in \mathcal{R}$ for all $i \geq 0$ s.t. $\omega_{q_0} \not\models \overline{\varphi}$, which implies Eq. \eqref{thm:sound_eq}. 

By Lemma~\ref{lemma: conservative labeling}, $L(q) \subseteq L(x)$ for $(x,q) \in \mathcal{R}$. By 
Proposition \ref{prop:trans}, for every transition $x' = f(x)$ if $(x,q) \in \mathcal{R}$ then $\exists q'$ such that $T(q, q') = 1$ and $(x',q') \in \mathcal{R}$. This satisfies the requirements for a sound abstraction.

We then prove the contrapositive, 
$$\omega_{x_0} \not\models \overline{\varphi} \implies \exists \omega_{q_0} \not\models \overline{\varphi}$$
That is, any counter example on System \eqref{true_dynamics} must also have a counter example in the NTS. 
A sound abstraction ensures $\exists \omega_{q_0}$ that remains in relation with $\omega_{x_0}$ for all $i \geq 0$.
Recall that $\overline{\varphi}$ has no negations at all, so $\neg \overline{\varphi}$ can be written in NNF where negations are applied to all propositions.

Assume $\omega_{x_0} \not\models \overline{\varphi}$, i.e. $\omega_{x_0} \models \neg \overline{\varphi}$. We begin with the base case $\overline{\varphi} := p$. Then $\omega_{x_0} \not\models p$ 
means 
$p \notin L(x_0)$ and since $L(q_i) \subseteq L(x_i)$ this implies $p \notin L(q_0)$ and therefore $\omega_{q_0} \models \neg p$. Hence satisfaction of negative propositions in $AP$ is preserved from System \eqref{true_dynamics} to $\mathcal{N}$. Since $\omega_{q_0}$ remains in relation with $\omega_{x_0}$, the satisfaction of negated propositions holds at subsequent states. By the semantics of LTL this satisfaction is applied inductively to Boolean and Temporal Operators.
Then $\omega_{x_0} \not\models \overline{\varphi}$ implies $\exists \omega_{q_0} \not\models \overline{\varphi}$. This establishes the contrapositive and hence $\Omega_q \models \overline{\varphi} \implies \omega_x \models \overline{\varphi}$.

\end{document}